\documentclass[twoside]{article}
\usepackage{graphicx}
\usepackage{bbold}
\usepackage{optidef}
\usepackage{multirow}

\usepackage{amsthm}
\usepackage{amsfonts,amssymb}
\theoremstyle{definition}
\newtheorem{defn}{Definition}
\theoremstyle{plain}

\usepackage[utf8]{inputenc}
\usepackage[T1]{fontenc}
\usepackage{optidef}
\usepackage{fixmath}
\usepackage{hyperref}
\usepackage{nicematrix}
\usepackage{xfrac} 
\usepackage{caption}
\usepackage{subcaption}
\usepackage{comment}
\usepackage{cancel}
\usepackage[ruled,linesnumbered]{algorithm2e}

\SetCommentSty{mycommfont}
\hypersetup{
    colorlinks=true,
    linkcolor=blue,
    filecolor=magenta,      
    urlcolor=black,
}

\usepackage[utf8]{inputenc} 
\usepackage[T1]{fontenc}    
\usepackage{hyperref}       
\usepackage{url}            
\usepackage{booktabs}       
\usepackage{amsmath}
\usepackage{amsfonts}       
\usepackage{nicefrac}       
\usepackage{microtype}      
\usepackage{lipsum}		
\usepackage{graphicx}
\usepackage{doi}

\theoremstyle{plain}
\newtheorem{theorem}{Theorem}[section]

\newtheorem{lemma}[theorem]{Lemma}
\newtheorem{corollary}[theorem]{Corollary}
\theoremstyle{definition}

\newtheorem{assumption}[theorem]{Assumption}
\theoremstyle{remark}

\usepackage[preprint]{aistats2027}
\usepackage[round]{natbib}

\begin{document}

%
\runningtitle{Accelerated Split Gibbs Sampling with Kinetic Langevin}

%

\twocolumn[

\aistatstitle{Kinetic Langevin Meets Split Gibbs: Accelerated Posterior Sampling for Imaging Inverse Problems with Diffusion Priors}

\aistatsauthor{ Dai Hai Nguyen \And Duc Dung Nguyen }

\aistatsaddress{Hokkaido University \And  Institute of Information Technology, VAST} ]

\begin{abstract}
  Split Gibbs sampling (SGS) is a popular framework for posterior sampling in Bayesian imaging inverse problems. It decouples a Gaussian data-fidelity term from a complex prior through an auxiliary variable, so the data variable is updated exactly and only the prior-side conditional is hard to sample. Existing samplers treat this conditional in one of two ways. Plug-and-play SGS runs a multi-step diffusion denoiser at every iteration, which is expensive and lacks non-asymptotic guarantees. Langevin-within-SGS takes cheap overdamped Langevin steps but needs many iterations. We propose RED-KLwSGS, which keeps the exact Gaussian update for the data variable and updates the auxiliary variable with underdamped (kinetic) Langevin diffusions driven by a one-shot denoising score, at the same per-iteration cost as Langevin-within-SGS. We prove non-asymptotic Wasserstein-2 convergence in continuous and discrete time for strongly log-concave priors. We also introduce Joint-RED-KLwSGS, which applies kinetic Langevin diffusions to both variables. Experiments with Denoising diffusion probabilistic models as diffusion priors on FFHQ and ImageNet datasets show faster convergence and high-quality image reconstruction.
\end{abstract}

\section{Introduction}
Image Restoration (IR) is an inverse problem that aims to recover an unknown image $\mathbf{x}\in\mathbb{R}^d$ from a degraded observation $\mathbf{y}\in\mathbb{R}^m$, related by the linear model $\mathbf{y}=\mathbf{A}\mathbf{x}+\mathbf{n}$, where $\mathbf{A}\in\mathbb{R}^{m\times d}$ is a degradation operator and $\mathbf{n}$ is a Gaussian random vector with covariance matrix $\mathbf{\Omega}^{-1}$ \citep{kaipio2005statistical}. We assume the noise covariance is diagonal,
$\mathbf{\Omega}^{-1}=\mathrm{diag}[\sigma_1^2,\dots,\sigma_m^2]$.
The likelihood is 
\begin{align}
\label{eqn:likelihoodfunction}
\begin{split}
    p(\mathbf{y}|\mathbf{x})&\propto \exp\left[-f(\mathbf{x},\mathbf{y}) \right], \\
    f(\mathbf{x},\mathbf{y})&=\frac{1}{2}\left( \mathbf{A}\mathbf{x}-\mathbf{y}\right)^\top \mathbf{\Omega}\left( \mathbf{A}\mathbf{x}-\mathbf{y}\right),
\end{split}
\end{align}
where $f:\mathbb{R}^{d}\times\mathbb{R}^m\rightarrow \mathbb{R}$ is the data-fidelity term
. Since IR is generally ill-posed, a prior is placed on $\mathbf{x}$,
\begin{equation}
\label{eqn:prior}
    p(\mathbf{x})\propto \exp\left[-\beta g(\mathbf{x}) \right],
\end{equation}
where $g:\mathbb{R}^{d}\rightarrow \mathbb{R}$ is a regularization potential, and $\beta>0$ is a regularization weight. Combining \eqref{eqn:likelihoodfunction} and \eqref{eqn:prior} via Bayes' rule gives the posterior
\begin{equation}
\label{eqn:posterior}
    \pi(\mathbf{x}):=p(\mathbf{x}|\mathbf{y})\propto \exp\left[-f(\mathbf{x},\mathbf{y}) -\beta g(\mathbf{x})\right].
\end{equation}
Much progress has concerned the choice of $g$. Classical hand-designed regularizers, such as total variation \citep{rudin1992nonlinear} or Sobolev penalties \citep{karl2005regularization}, 
encode simple structural assumptions and struggle to represent natural images. This motivated implicit, denoiser-driven priors. Plug-and-play (PnP) methods \citep{venkatakrishnan2013plug} use an off-the-shelf denoiser as a proxy for the proximal operator of $g$ inside a splitting scheme such as ADMM \citep{afonso2010augmented} or HQS \citep{geman1995nonlinear}. Regularization by denoising (RED) instead builds an explicit image-adaptive prior directly from a denoiser, and has been reported to outperform PnP in practice \citep{romano2017little}.


Estimating $\mathbf{x}$ from $\mathbf{y}$ can be done by computing the maximum-a-posterior (MAP) estimate via solving the corresponding minimization problem. While computationally efficient, the MAP estimate reduces the posterior $p(\mathbf{x}|\mathbf{y})$ to a single point and therefore cannot quantify the uncertainty. Markov chain Monte Carlo (MCMC) gives access to the full posterior \citep{brooks1998markov}, but standard samplers are costly in imaging dimensions, especially when $g$ is only implicitly defined through a denoiser as in PnP or RED \citep{faye2024regularization}. Mirroring the variable splitting strategies that have proven effective in optimization (e.g., ADMM, HQS), split Gibbs samplers (SGS) address this difficulty by introducing an auxiliary variable $\mathbf{z}$ and targeting an augmented distribution whose conditionals are individually easier to sample: one governed solely by the data-fidelity term $f$, typically Gaussian and available in closed form, and the other governed by the prior term $g$, which retains the flexibility to accommodate implicit, denoiser-based prior \citep{vono2019split, rendell2020global}.


A challenge for SGS when applied to PnP and RED is that sampling the prior-side conditional is not straightforward, because the potential $g$ is defined through a denoiser or generative model. To address this, both PnP-SGS \citep{coeurdoux2024plug} and RED-LwSGS \citep{faye2024regularization} use a pretrained deep generative model, namely a denoising diffusion probabilistic model, as a stochastic denoiser. What differs is how each uses it. PnP-SGS draws the auxiliary variable $\mathbf{z}$ by producing a denoised sample of the noisy $\mathbf{x}$, which requires running the reverse process through many intermediate timesteps. Thus, each Gibbs iteration cost many denoiser network evaluations. RED-LwSGS instead performs a Langevin Monte Carlo (LMC) step to update $\mathbf{z}$, which requires only a coarse gradient direction. This costs exactly one network evaluation per iteration, so each step is far cheaper, but the resulting sampler needs many more Gibbs iterations to converge. Subsection \ref{section:background_comparison} and Appendix \ref{app:diffusionmodels} give details.

The LMC update discretizes the \textit{overdamped} Langevin diffusions, given by the first-order stochastic differential equation (SDE) \citep{eberle2016reflection,durmus2017nonasymptotic}, which has been widely used in statistics and machine learning to sample from complex distributions \citep{nguyen2024moreau,welling2011bayesian,nguyen2023mirror,nguyen2023wasserstein,nguyen2025multiple}. An alternative is the class of \textit{underdamped} or \textit{kinetic} Langevin diffusions \citep{cheng2018underdamped}, which are the second-order SDEs and converge faster to the target distributions than their overdamped counterpart \citep{cheng2018underdamped,nguyen2026accelerated,wang2022accelerated}. In this work, we aim to develop efficient sampling methods for posterior \eqref{eqn:posterior} based on SGS and the kinetic Langevin diffusions, exploiting their faster convergence. Our contributions are as follows:
\begin{itemize}
    \item We introduce the Kinetic-Langevin-within-Split Gibbs  sampler under RED (\textbf{RED-KLwSGS}), based on the variable-splitting technique of SGS: it keeps the exact Gaussian update for the data variable $\mathbf{X}$ and updates the auxiliary variable $\mathbf{Z}$ with kinetic Langevin diffusions.  We also introduce \textbf{Joint-RED-KLwSGS}, which applies kinetic Langevin diffusions to both variables. 
    \item We provide \textbf{theoretical guarantees} of RED-KLwSGS in both continuous time and its discretization.
    \item We \textbf{apply the proposed methods to solve three IR tasks}, namely Gaussian deblurring, motion deblurring, and super-resolution. Numerical experiments on FFHQ and ImageNet datasets show faster convergence and high-quality image reconstruction for the proposed methods.
\end{itemize}

\section{Related Work}
In this section, we first present background on RED. We then introduce two split Gibbs samplers, PnP-SGS and RED-LwSGS, for sampling from the posterior \eqref{eqn:posterior}.

\subsection{Regularization-by-denoising (RED)}
RED defines an explicit, image-adaptive Laplacian-based potential
\begin{equation}
    g_{\mathrm{RED}}(\mathbf{x})=\frac{1}{2}\mathbf{x}^\top\left(\mathbf{x}-D_{\nu}(\mathbf{x}) \right),
\end{equation}
where $D_{\nu}:\mathbb{R}^d\rightarrow\mathbb{R}^d$ is a denoiser with parameter $\nu$ controlling the denoising strength, designed to remove white Gaussian noise. Although RED offers significant flexibility in the choice of denoisers, it requires the following conditions on $D_{\nu}$, referred to as RED conditions:

\textbf{(C1) Local homogeneity:} $D_{\nu}\left( (1+\varepsilon)\mathbf{x}\right) = (1+\varepsilon)D_{\nu}(\mathbf{x})$, for small $\varepsilon>0$.

\textbf{(C2) Differentiability:} $D_{\nu}$ is differentiable, with Jacobian $\nabla D_{\nu}$.

\textbf{(C3) Jacobian symmetry:}  $\nabla D_{\nu}(\mathbf{x})^\top = \nabla D_{\nu}(\mathbf{x})$.

\textbf{(C4) Strong passivity:} The spectra radius of the Jacobian is at most $1$, i.e., $\eta\left( \nabla D_{\nu}(\mathbf{x})\right)\leq 1$, where $\eta(\mathbf{M})=\max\left\{|\lambda|:\lambda \text{ is an eigenvalue of the matrix }\mathbf{M} \right\}$. Local homogeneity (C1) implies $D_{\nu}(\mathbf{x})=\nabla D_{\nu}(\mathbf{x})\mathbf{x}$, and Jacobian symmetry (C3) gives $\nabla D_{\nu}(\mathbf{x})^\top\mathbf{x}  = \nabla D_{\nu}(\mathbf{x})\mathbf{x}=D_{\nu}(\mathbf{x})$ \citep{romano2017little}. Therefore the gradient of $g_{\mathrm{RED}}$ is the denoising residual,
\begin{align}
\label{eqn:nablagRED}
\begin{split}
    \nabla g_{\mathrm{RED}}(\mathbf{x})&=\mathbf{x}- D_{\nu}(\mathbf{x}), \\
    \nabla^2 g_{\mathrm{RED}}(\mathbf{x})&=\mathbf{I}_d-\nabla D_{\nu}(\mathbf{x}).
\end{split}
\end{align}
Because $\nabla D_{\nu}(\mathbf{x})$ is symmetric (C3) with spectral radius at most $1$ (C4), its eigenvalues lie in $[-1, 1]$. Two consequences follow: first, $\lVert D_{\nu}(\mathbf{x}) \rVert_2=\lVert\nabla D_{\nu}(\mathbf{x})\mathbf{x} \rVert_2 \leq \lVert\mathbf{x} \rVert_2$, where the equality uses (C1) and the inequality uses (C3)-(C4); second, the eigenvalues of $\nabla^2 g_{\mathrm{RED}}$ lie in $[0,2]$, so $g_{\mathrm{RED}}$ is convex and $\lVert \nabla^2 g_{\mathrm{RED}}\rVert_{\mathrm{op}}\leq 2$, where $\lVert\cdot\rVert_{\mathrm{op}}$ is the operator norm.

The gradient identity \eqref{eqn:nablagRED} is the key property of RED: the gradient of the prior is obtained from a single denoiser evaluation, without differentiating the denoiser. This allows powerful denoisers, such as those based on deep neural networks \citep{ho2020denoising,kingma2021variational,zhang2021plug}, to be embedded in the prior. Building on this, one defines a prior distribution from the RED potential \citep{faye2024regularization}
\begin{equation}
\label{eqn:REDprior}
    p_{\mathrm{RED}}(\mathbf{x})\propto \exp\left[ -\frac{\beta}{2}\mathbf{x}^\top\left( \mathbf{x}-D_{\nu}(\mathbf{x})\right)\right].
\end{equation}
Combining the likelihood with the RED prior gives the RED posterior 
\begin{equation}
    p_{\mathrm{RED}}(\mathbf{x}|\mathbf{y})\propto \exp\left[-f(\mathbf{x},\mathbf{y}) -\beta g_{\mathrm{RED}}(\mathbf{x})\right].
\end{equation}

\subsection{Split Gibbs Sampling (SGS)}
Sampling efficiently from the posterior $p(\mathbf{x}|\mathbf{y})$ is not straightforward because its potential involves the denoiser $D_{\nu}$. Inspired by optimization, asymptotically exact data augmentation (AXDA) \citep{vono2020asymptotically} uses variable splitting: it introduces an auxiliary variable $\mathbf{z}\in\mathbb{R}^d$, and targets the augmented distribution
\begin{align}
\label{eqn:augmentedposterior}
    \begin{split}
        &\pi_{\rho}(\mathbf{x}, \mathbf{z}) = p(\mathbf{x}, \mathbf{z}|\mathbf{y};\rho^2)\\
        &\propto \exp\left[ -f(\mathbf{x},\mathbf{y})-\beta g(\mathbf{z})-\frac{1}{2\rho^2}\lVert\mathbf{x}-\mathbf{z} \rVert^2_2\right],
    \end{split}
\end{align}
where the positive parameter $\rho$ controls the coupling between $\mathbf{x}$ and $\mathbf{z}$. As shown in \cite{vono2020asymptotically}, for a large class of coupling kernels including the quadratic one, the $\mathbf{x}$-marginal of $\pi_{\rho}$ converges to the posterior $\pi$ as $\rho^2\rightarrow 0$. Instead of sampling $\pi$ directly, SGS \citep{vono2019split} samples the augmented distribution $\pi_{\rho}$ by Gibbs steps. The two conditionals are 
\begin{align}
\label{eqn:conditional_x}
    p(\mathbf{x}|\mathbf{z};\mathbf{y}, \rho^2) &\propto \exp\left[ -f(\mathbf{x},\mathbf{z})-\frac{1}{2\rho^2}\lVert\mathbf{x}-\mathbf{z} \rVert^2_2\right],\\
\label{eqn:conditional_z}
    p(\mathbf{z}|\mathbf{x}; \rho^2) &\propto \exp\left[ -\beta g(\mathbf{z})-\frac{1}{2\rho^2}\lVert\mathbf{x}-\mathbf{z} \rVert^2_2\right].
\end{align}
SGS alternatively samples from 
\eqref{eqn:conditional_x} and \eqref{eqn:conditional_z}, which generates samples asymptotically distributed according to \eqref{eqn:augmentedposterior}. Sampling from these conditionals decouples the likelihood potential $f(\cdot, \mathbf{y})$ from the prior potential $g(\cdot)$. As a consequence, SGS inherits advantages of its deterministic counterparts (i.e., HQS and ADMM), such as simpler implementations, faster convergences and the possibility of distributed computation.

\subsection{PnP-SGS and RED-LwSGS}
\label{section:background_comparison}
Sampling from the $\mathbf{x}$-conditional \eqref{eqn:conditional_x} is equivalent to sampling the posterior of the original problem with the same likelihood 
$f(\cdot, \mathbf{y})$, but with an isotropic Gaussian prior $\mathcal{N}(\mathbf{z}, \rho^2\mathbf{I}_d)$ in place of $p(\mathbf{x})$. This is much easier than sampling $\pi$ directly. For the quadratic potential $f(\mathbf{x},\mathbf{y})=1/2\left(\mathbf{A}\mathbf{x}-\mathbf{y}\right)^\top\mathbf{\Omega}\left( \mathbf{A}\mathbf{x}-\mathbf{y}\right)$, which covers a large family of imaging inverse problems, such as deblurring and super-resolution, this conditional is
\begin{equation}
    p(\mathbf{x}|\mathbf{z};\mathbf{y}, \rho^2)=\mathcal{N}(\mathbf{x}; M(\mathbf{z}), \mathbf{Q}^{-1}),
\end{equation}
with precision matrix and mean 
\begin{align}
\label{eqn:mean_covariance}
    \begin{split}
        \mathbf{Q}&= \mathbf{A}^\top\mathbf{\Omega}\mathbf{A}+\rho^{-2}\mathbf{I}_d,\\
        M(\mathbf{z})&=\mathbf{Q}^{-1}\left(\mathbf{A}^\top\mathbf{\Omega}\mathbf{y}+\rho^{-2}\mathbf{z} \right).
    \end{split}
\end{align}
\noindent
The two methods below differ only in how they handle the remaining prior-side conditional \eqref{eqn:conditional_z}.

\noindent
\textbf{Plug-and-Play SGS (PnP-SGS)}. Sampling the prior-side conditional \eqref{eqn:conditional_z} is not straightforward because the potential $g$ is defined through a denoiser or generative model. \cite{coeurdoux2024plug} observe that \eqref{eqn:conditional_z} is a posterior of a Bayesian denoising problem: recovering $\mathbf{z}$ from the noisy observation $\mathbf{x}=\mathbf{z}+\rho \mathbf{n}$, with $\mathbf{n}\sim \mathcal{N}(\mathbf{0}, \mathbf{I}_d)$, under the prior $e^{-\beta g}$. They therefore do not sample \eqref{eqn:conditional_z} directly, but use a pretrained deep generative model, a denoising diffusion probabilistic model (DDPM) \citep{ho2020denoising,kingma2021variational}, as a stochastic denoiser. Each $\mathbf{z}$-update starts the model's reverse diffusion at the timestep matching the noise level $\rho$ and returns a sample of the denoising posterior.

\noindent
\textbf{Langevin-within-SGS  under RED (RED-LwSGS)}. The second approach takes $g=g_{\mathrm{RED}}$. When the denoiser satisfies (C1)-(C4), \eqref{eqn:nablagRED} gives
\begin{equation*}
    \nabla p(\mathbf{z}|\mathbf{x};\rho^2)=-\beta (\mathbf{z}-D_{\nu}(\mathbf{z}))-\frac{1}{\rho^2}(\mathbf{z}-\mathbf{x}),
\end{equation*}
so the score of \eqref{eqn:conditional_z} needs one denoiser evaluation and no differentiation of the denoiser. \cite{faye2024regularization} exploit this by replacing the exact draw from \eqref{eqn:conditional_z} with one unadjusted Langevin Monte Carlo (LMC) step:
\begin{equation}
    \label{eqn:lmcupdate}
    \mathbf{Z}^{(k+1)}=\mathbf{Z}^{(k)}+h \nabla \log p(\mathbf{Z}^{(k)}|\mathbf{X}^{(k)};\rho^2)+\sqrt{2h} \mathbf{N}^{(k)},
\end{equation}
where $h>0$ is a fixed step size, $\left\{ \mathbf{N}^{(k)}\right\}_{k\in \mathbb{N}}$ are i.i.d.\ standard Gaussian vectors in $\mathbb{R}^d$, and $\mathbf{X}^{(k)}\sim p(\mathbf{x}|\mathbf{Z}^{(k)};\mathbf{y}, \rho^2)$ is a sample from \eqref{eqn:conditional_x} given $\mathbf{Z}^{(k)}$. 
Substituting the score gives
\begin{align}
\begin{split}
     \label{eqn:lmcupdate}
    \mathbf{Z}^{(k+1)}&=\left(1- h \beta -\frac{h}{\rho^2} \right)\mathbf{Z}^{(k)} + \frac{h}{\rho^2}\mathbf{X}^{(k)} \\
    &+h\beta D_{\nu}(\mathbf{Z}^{(k)})
    +\sqrt{2h}\mathbf{N}^{(k)}.
\end{split}
\end{align}

\noindent
\textbf{PnP-SGS versus RED-LwSGS}. Both methods use the same pretrained diffusion network. What differs is how each uses it. PnP-SGS draws $\mathbf{z}$ from \eqref{eqn:conditional_z} by producing a denoised sample of the noisy $\mathbf{x}$. Diffusion models are trained to remove only small increments of noise reliably, so a high-fidelity sample at a given noise level requires running the reverse process through many intermediate timesteps. Each Gibbs iteration therefore costs many network evaluations. RED-LwSGS instead computes only a coarse gradient direction for $\mathbf{z}$, obtained from a single one-shot estimate of the clean image by Tweedie's formula \citep{efron2011tweedie}. This costs exactly one network evaluation per iteration, so each step is far cheaper, but the resulting sampler needs many more Gibbs iteration to converge.

We illustrate this trade-off on motion deblurring of a randomly selected FFHQ image \citep{karras2019style}. The observation $\mathbf{y}$ is generated by applying a motion-blur kernel to the clean image $\mathbf{x}$ (see Section \ref{sec:exps} for details). We measure the wall-clock time per iteration and the number of iterations each method needs to reach a fixed target PSNR of 26.5 dB against the ground truth. Table \ref{tab:pnp_vs_lwsgs_full} below reports the results. The per-iteration cost of RED-LwSGS (0.087) is about 24 times lower than that of PnP-SGS (2.121 s), but RED-LwSGS needs far more iterations (600 versus 22). The two effects roughly cancel in total time (46.64 s versus 52.2 s), so neither method is clearly faster overall. This motivates our proposed RED-KLwSGS, which keeps the low per-iteration cost of RED-LwSGS (0.091 s) while needing far fewer iterations (350) to reach the similar target, and therefore less total time (31.5 s).

\begin{table}[h]
\centering
\caption{Time per iteration, iterations, and total time to reach PSNR $=26.5$\,dB on an FFHQ image for motion deblurring task.}
\label{tab:pnp_vs_lwsgs_full}
\renewcommand{\arraystretch}{1.3}
\resizebox{\columnwidth}{!}{%
\begin{tabular}{lccc}
\toprule
Method & Time/iter (s) & Iterations & Total time (s) \\
\midrule
PnP-SGS & 2.121 & 22  & 46.64\\
RED-LwSGS   & 0.087 & 600 & 52.2 \\
\textbf{RED-KLwSGS (ours)} & \textbf{0.091} & \textbf{350} & \textbf{31.5} \\
\bottomrule
\end{tabular}%
}
\end{table}

\section{Proposed Methods}\label{sec:kLwSGS}

\subsection{RED-KLwSGS: Kinetic Lavengin-within-Split Gibbs under RED}
The LMC update  in \eqref{eqn:lmcupdate} is a discretization of the \textit{overdamped}) Langevin diffusions for the auxiliary variable $\mathbf{Z}$, coupled to the exact Gaussian draw of $\mathbf{X}$, given by 
 \begin{align}
 \begin{split}
     d\mathbf{Z}_t &= \nabla \log p(\mathbf{Z}_t|\mathbf{X}_t;\rho^2) + \sqrt{2}d\mathbf{B}_t,\\
     \mathbf{X}_t &= M(\mathbf{Z}_t) + \mathbf{Q}^{-\frac{1}{2}}\boldsymbol{\epsilon}_t,
\end{split}
 \end{align}
 where $(\mathbf{B}_t)_{t\geq 0}$ is a Brownian motion, and $(\boldsymbol{\epsilon_t})_{t\geq 0}$ are standard Gaussian noises. This is an extension of the classical overdamped Langevin diffusions \citep{eberle2016reflection,durmus2017nonasymptotic} to coupled variables. An alternative to the overdamped Langevin diffusions is the class of kinetic, or underdamped Langevin diffusions \citep{cheng2018underdamped}. These are second-order SDEs over position and momentum variables. In log-concave settings, they are known to converge faster than their overdamped counterparts \citep{cheng2018underdamped,nguyen2026accelerated,wang2022accelerated}. We apply this to the prior-side conditional \eqref{eqn:conditional_z}, keeping the exact update for $\mathbf{x}$. The resulting RED-KLwSGS diffusion is
 \begin{align}
 \label{eqn:kineticflow}
     \begin{split}
         d\mathbf{Z}_t&=\mathbf{V}_tdt,  \mathbf{X}_t = M(\mathbf{Z}_t) + \mathbf{Q}^{-\frac{1}{2}}\boldsymbol{\epsilon}_t, \\
         d\mathbf{V}_t&= -\gamma \mathbf{V}_t dt + u \nabla\log p(\mathbf{Z}_t|\mathbf{X}_t; \rho^2)dt + \sqrt{2\gamma u}d \mathbf{B}_t,\\
     \end{split}
 \end{align}
 where $\gamma>0$ is the friction coefficient and $u>0$ scales the gradient, acting as an inverse mass. The stationary distribution of \eqref{eqn:kineticflow} is $\Pi_\rho:=\pi_{\rho}\otimes\mathcal{N}(\mathbf{0}, u\mathbf{I}_d)$. We discretize \eqref{eqn:kineticflow} with step size $h$, giving the following RED-KLwSGS sampler
  \begin{align}
 \label{eqn:discretupdate}
     \begin{split}
        &\mathbf{V}^{(k+1)}=(1-h\gamma)\mathbf{V}^{(k)}\\
        &+ h u \nabla \log p(\mathbf{Z}^{(k)}|\mathbf{X}^{(k)}; \rho^2)+\sqrt{2 \gamma h u} \mathbf{N}^{(k)},\\
        &\mathbf{Z}^{(k+1)}= \mathbf{Z}^{(k)} + h \mathbf{V}^{(k+1)},\\
        &\mathbf{X}^{(k+1)}=M(\mathbf{Z}^{(k+1)}) + \mathbf{Q}^{-\frac{1}{2}}\boldsymbol{\epsilon}^{(k)},
     \end{split}
 \end{align}
 where $(\mathbf{N}^{(k)})_{k\in\mathbb{N}}$ and $(\boldsymbol{\epsilon}^{(k)})_{k\in\mathbb{N}}$ are independent sequences of standard Gaussian vectors. Using the score from \eqref{eqn:conditional_z} and the property \eqref{eqn:nablagRED}, \eqref{eqn:discretupdate} reads explicitly as
 \begin{align}
 \label{eqn:kineticalgorithm}
     \begin{split}
        &\mathbf{V}^{(k+1)}=(1-h\gamma)\mathbf{V}^{(k)} -hu \left( \beta + \frac{1}{\rho^2}\right)\mathbf{Z}^{(k)} \\
        &+\frac{hu}{\rho^2}\mathbf{X}^{(k)}+ hu\beta D_{\nu}(\mathbf{Z}^{(k)})
        + \sqrt{2\gamma h u}\mathbf{N}^{(k)},\\
        &\mathbf{Z}^{(k+1)}= \mathbf{Z}^{(k)} + h \mathbf{V}^{(k+1)},\\
        &\mathbf{X}^{(k+1)}=M(\mathbf{Z}^{(k+1)}) + \mathbf{Q}^{-\frac{1}{2}}\boldsymbol{\epsilon}^{(k)},
     \end{split}
 \end{align}
 \noindent
 \textbf{Cost per iteration.} An iteration of \eqref{eqn:kineticalgorithm} needs one denoiser evaluation, $D_{\nu}(\mathbf{Z}^{(k)})$, and one exact Gaussian draw for $\mathbf{X}$, both identical to RED-LwSGS, so the cost per iteration is unchanged. The only new state is the velocity $\mathbf{V}$, which cost $\mathcal{O}(d)$ memory and arithmetic.






\subsection{Theoretical analysis}
\label{subsection:theoreticalanalysis}
This subsection provides theoretical guarantees for RED-KLwSGS. The analysis holds for any potential $g$ satisfying the two assumptions below. We then show that $g_{\mathrm{RED}}$ satisfies them, with explicit constants, under the conditions (C1)-(C4) and a contractive denoiser.
\begin{assumption}[Smoothness]
    The potential $g$ is twice differentiable and there is $M_{g}>0$ such that $\forall \mathbf{z}\in\mathbb{R}^d$, $\lVert \nabla^2g(\mathbf{z})\rVert_\mathrm{op}\leq M_{g}$.
\end{assumption}
Under the RED conditions (C1) and (C3), the gradient of $g_{\mathrm{RED}}$ is given by \eqref{eqn:nablagRED}. This implies that $g_{\mathrm{RED}}$ is twice continuously differentiable with Hessian matrix $\nabla^2 g_{\mathrm{RED}}(\mathbf{z})=\mathbf{I}_d-\nabla D_{\nu}(\mathbf{z})$. Furthermore, thanks to conditions (C3) and (C4), $\lVert \nabla^2g(\mathbf{z})\rVert_\mathrm{op}\leq 2$.

\begin{assumption}[Strong convexity]
    The potential $g(\cdot)$ is $m_g$-strongly convex, i.e., $\forall \mathbf{z}\in\mathbb{R}^d$, $m_g \mathbf{I}_d\preceq \nabla^2 g(\mathbf{z})$. 
\end{assumption}
In the RED framework, a sufficient condition for the strong convexity of $g_{\mathrm{RED}}$ is that the denoiser $D_{\nu}$ is contractive, i.e., $\forall \mathbf{z}_1, \mathbf{z}_2\in\mathbb{R}^d$, $\lVert D_{\nu}(\mathbf{z}_1)- D_{\nu}(\mathbf{z}_2) \rVert\leq \epsilon\lVert \mathbf{z}_1- \mathbf{z}_2\rVert$, for some Lipschitz constant $\epsilon<1$. This can be enforced by explicitly including Lipschitz-constraining regularization term in the denoiser's training loss \citep{faye2024regularization}. Under this constraint, $g_{\mathrm{RED}}$ is $m_g$-strongly convex with $m_g =1-\epsilon>0$.

\begin{figure*}[t]
    \centering
    \includegraphics[width=0.99\textwidth]{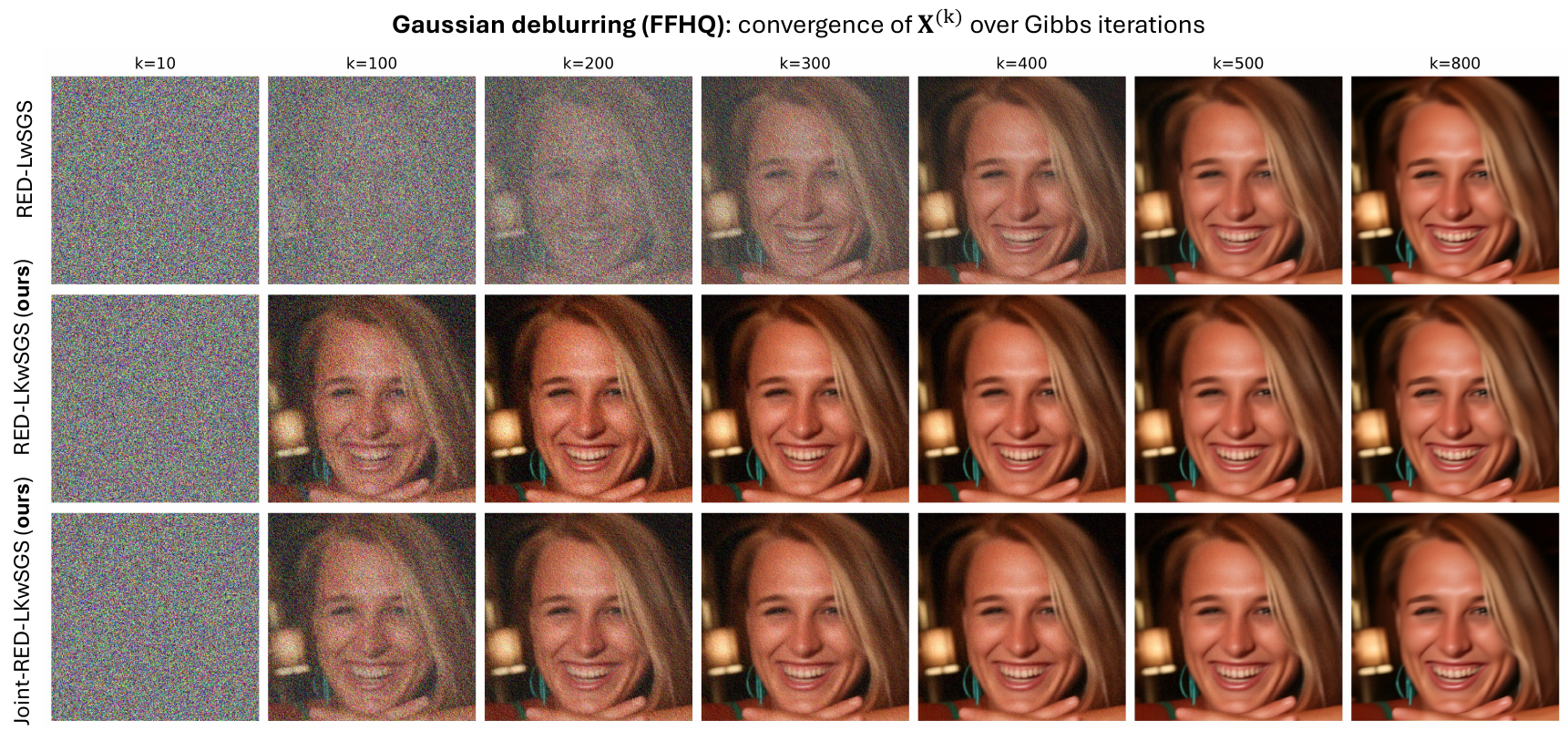}
    \caption{Visual convergence of $\mathbf{X}^{(k)}$ on Gaussian deblurring of an FFHQ image, over 800 Gibbs iterations. Rows show RED-LwSGS, RED-KLwSGS (ours), and Joint-RED-KLwSGS (ours), each started from the same initialization; columns show iterations $k=10,100,200,300,400,500,800$.}
    \label{fig:convergence_motion}
\end{figure*}

\begin{figure*}[t]
    \centering
    \includegraphics[width=0.99\textwidth]{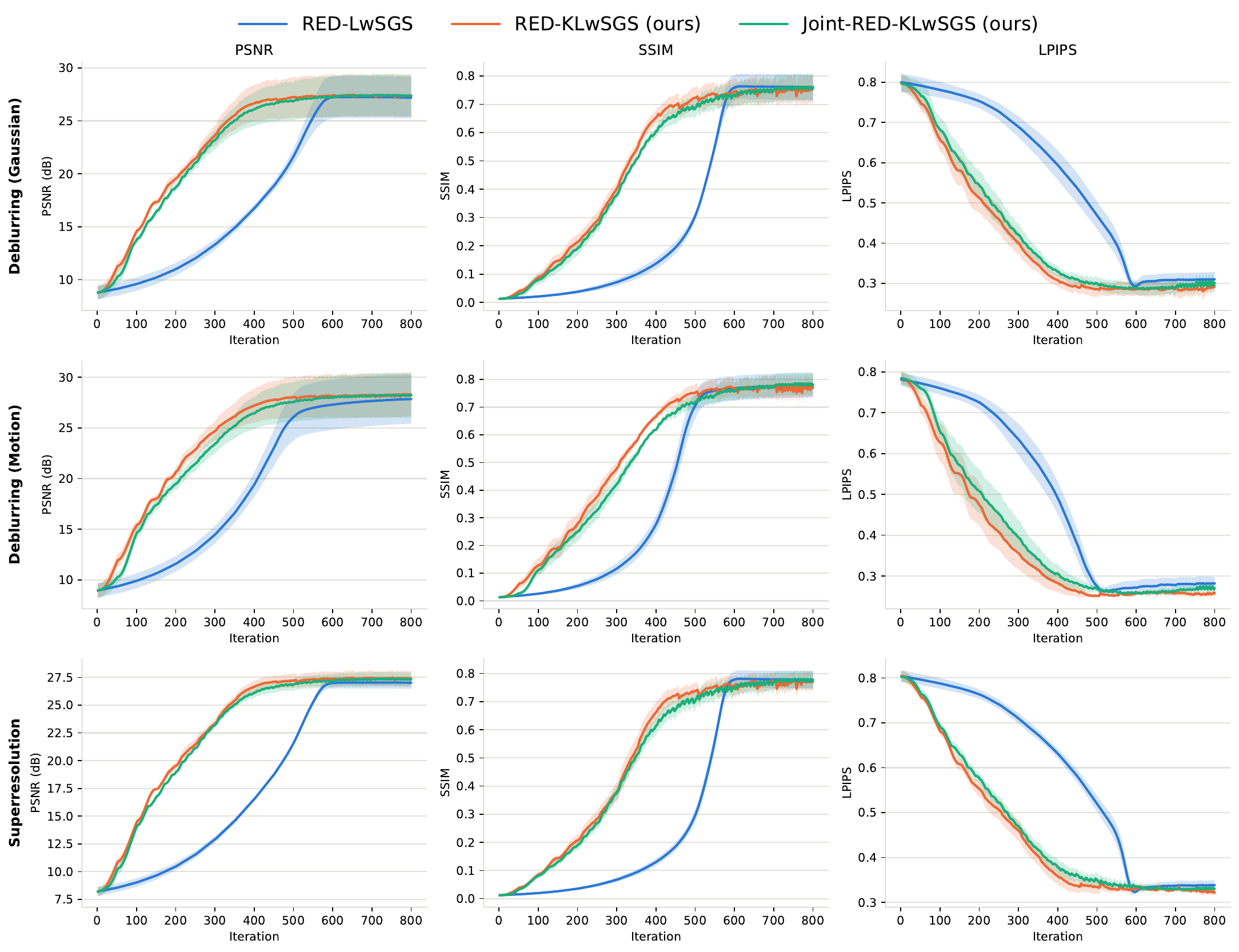}
    \caption{Convergence of RED-LwSGS (baseline), RED-KLwSGS (ours), and
    Joint-RED-KLwSGS (ours) for Gaussian deblurring, motion deblurring,
    and super-resolution (rows), on ten FFHQ images, measured by PSNR, SSIM, and LPIPS (columns).
    Solid lines show the mean over independent runs per method; shaded
    bands show $\pm 1$ standard deviation.}
    \label{fig:convergence_all_tasks}
\end{figure*}

\subsubsection{Convergence of the continuous-time process}
We first consider the continuous-time RED-KLwSGS diffusion, given by the SDE \eqref{eqn:kineticflow}, with an initial condition $(\mathbf{X}_0, \mathbf{Z}_0, \mathbf{V}_0)\sim \mu_0$ for some distribution $\mu_0$ on $\mathbb{R}^{3d}$. Let $\mu_t$ denote the law of $(\mathbf{X}_t, \mathbf{Z}_t, \mathbf{V}_t)$ and let $\mathbf{\Phi}_t$ denote the operator mapping $\mu_0$ to $\mu_t$, i.e., $\mathbf{\Phi}_t\mu_0=\mu_t$. We establish the contraction of ~\eqref{eqn:kineticflow} as follows.

\begin{theorem}
\label{theorem1}
    Let $\mu_0$ and $\widetilde{\mu}_0$ be arbitrary distributions on $\mathbb{R}^{3d}$, with $(\mathbf{X}_0, \mathbf{Z}_0, \mathbf{V}_0)\sim \mu_0$, $(\widetilde{\mathbf{X}}_0, \widetilde{\mathbf{Z}}_0, \widetilde{\mathbf{V}}_0)\sim \widetilde{\mu}_0$. Let $\mu_t=\mathbf{\Phi}_t \mu_0, \widetilde{\mu}_t=\mathbf{\Phi}_t \widetilde{\mu}_0$ denote the laws of $(\mathbf{X}_t, \mathbf{Z}_t, \mathbf{V}_t)$, $(\widetilde{\mathbf{X}}_t, \widetilde{\mathbf{Z}}_t, \widetilde{\mathbf{V}}_t)$, which evolve according to the SDE \eqref{eqn:kineticflow}, with the respective initial conditions $(\mathbf{X}_0, \mathbf{Z}_0, \mathbf{V}_0)$, $(\widetilde{\mathbf{X}}_0, \widetilde{\mathbf{Z}}_0, \widetilde{\mathbf{V}}_0)$. Set $u=(\beta M_g + \rho^{-2})^{-1}$ and $\gamma = 2$. Then for every $t\geq 0$,
    \begin{align}
    \begin{split}
        \mathcal{W}_2^2\left(\mu_t, \widetilde{\mu}_t \right) 
        \leq C e^{-\frac{t}{\kappa}}\mathcal{W}^2_2\left(\mu_0, \widetilde{\mu}_0\right),
    \end{split}
    \end{align}
    where 
    $C=7\left(1 +\rho^{-4}\lVert \mathbf{Q}^{-1}\rVert^2_{\mathrm{op}} \right)$ and $\kappa=(\beta M_g + \rho^{-2})/(\beta m_g)> 1$.
\end{theorem}
The proof of Theorem \ref{theorem1} is given in Appendix \ref{prooftheorem1}. This directly yields exponential convergence to the stationary distribution $\Pi_\rho$.
\begin{corollary}
\label{cor:continuous-convergence}
    For any distribution $\mu$ on $\mathbb{R}^{3d}$,
     \begin{equation}
        \mathcal{W}^2_2\left(\mathbf{\Phi}_t\mu, \Pi_{\rho}\right)\leq C e^{-\frac{t}{\kappa}}\mathcal{W}^2_2\left(\mu, \Pi_{\rho}\right).
    \end{equation}
\end{corollary}

\subsubsection{Convergence of the discrete-time process}
We next consider the discrete-time RED-KLwSGS update \eqref{eqn:discretupdate}. Let $\mu^{(k)}$ denote the law of $\left(\mathbf{X}^{(k)}, \mathbf{Z}^{(k)}, \mathbf{V}^{(k)}\right)$, and $\bar{\mathbf{\Phi}}_h$ the operator mapping $\mu^{(k)}$ to $\mu^{(k+1)}$. We establish the contraction of \eqref{eqn:discretupdate} as follows.

\begin{theorem}
\label{theorem2}
    Let $\mu^{(0)}$ and $\widetilde{\mu}^{(0)}$ be arbitrary distributions on $\mathbb{R}^{3d}$, with $(\mathbf{X}^{(0)}, \mathbf{Z}^{(0)}, \mathbf{V}^{(0)})\sim \mu^{(0)}$, $(\widetilde{\mathbf{X}}^{(0)}, \widetilde{\mathbf{Z}}^{(0)}, \widetilde{\mathbf{V}}^{(0)})\sim \widetilde{\mu}^{(0)}$. Let $\mu^{(k+1)}=\bar{\mathbf{\Phi}}_h \mu^{(k)}$, $\widetilde{\mu}^{(k+1)}=\bar{\mathbf{\Phi}}_h \widetilde{\mu}^{(k)}$ denote the laws of $(\mathbf{X}^{(k+1)}, \mathbf{Z}^{(k+1)}, \mathbf{V}^{(k+1)})$, $(\widetilde{\mathbf{X}}^{(k+1)}, \widetilde{\mathbf{Z}}^{(k+1)}, \widetilde{\mathbf{V}}^{(k+1)})$, given by the update \eqref{eqn:discretupdate}. Set $u=(\beta M_g + \rho^{-2})^{-1}$ and $\gamma = 2$. Then for $h\in\left(0,1/2 \right]$ and any $k\in\mathbb{N}$,
    
    \begin{align}
    \begin{split}
        &\mathcal{W}^2_2\left(\mu^{(k)}, \widetilde{\mu}^{(k)}\right)=\mathcal{W}^2_2\left(\bar{\mathbf{\Phi}}_h^k\mu^{(0)}, \bar{\mathbf{\Phi}}_h^k\widetilde{\mu}^{(0)}\right)\\
        &\leq C \left( 1-\frac{h}{\kappa}+\frac{h^2}{\kappa^2}\right)^k\mathcal{W}^2_2\left(\mu^{(0)}, \widetilde{\mu}^{(0)}\right).
     \end{split}
    \end{align} 
\end{theorem}
The proof of Theorem \ref{theorem2} is given in Appendix \ref{prooftheorem2}. This directly yields convergence to the discrete-time stationary distribution.

\begin{corollary}
\label{cor:discrete-convergence}
    The map $\bar{\mathbf{\Phi}}_h$ has a unique stationary law $\Pi_{\rho h}$, and for any $\mu$ on $\mathbb{R}^{3d}$,
     \begin{equation}
        \mathcal{W}^2_2\left(\bar{\mathbf{\Phi}}_h^k\mu, \Pi_{\rho  h}\right)\leq C \left( 1-\frac{h}{\kappa}+\frac{h^2}{\kappa^2}\right)^k\mathcal{W}_2^2\left(\mu, \Pi_{\rho h}\right).
    \end{equation}
\end{corollary}
\subsubsection{Discretization bias and algorithmic convergence}
Having established convergence of the discrete-time process, we analyze the bias between $\Pi_{\rho}$ and  $\Pi_{\rho h}$.

\begin{lemma}
\label{lemma:bias}
Set $u=(\beta M_g + \rho^{-2})^{-1}$ and $\gamma=2$. For $h\in (0, 1/2]$, there is a constant $C_{\mathrm{bias}}$, independent of $h$, such that $\mathcal{W}^2_2(\Pi_{\rho}, \Pi_{\rho h}) \leq C_{\mathrm{bias}} h.$
\end{lemma}
The proof of Lemma \ref{lemma:bias} is given in Appendix \ref{proofofbiaslemma}. Based on this lemma, we establish the convergence of RED-KLwSGS algorithm as follows.

\begin{theorem}
\label{theorem:convergence}
    Let $\mu$ be an arbitrary distribution on $\mathbb{R}^{3d}$ and $\delta > 0$.
    Choose 
    \begin{equation*}
        h\leq \min \left\{\frac{1}{2}, \frac{\delta^2}{4 C_{\mathrm{bias}}} \right\}, 
    \end{equation*}
    and
    \begin{equation*}
        k \geq \frac{2 \ln (2\sqrt{C}(\mathcal{W}_2(\mu, \Pi_\rho)+\delta/2)/\delta)}{\ln \left(1/\left(1-\frac{h}{\kappa} + \frac{h^2}{\kappa^2} \right)\right)}.
    \end{equation*}
    Then $\mathcal{W}_2 (\bar{\boldsymbol{\Phi}}_h^{k}\mu, \Pi_{\rho})\leq \delta$.
\end{theorem}
The proof of Theorem \ref{theorem:convergence} is given in Appendix \ref{proof:convergence}

\subsection{Joint-RED-KLwSGS: applying kinetic Langevin diffusions to both variables}
The exact Gaussian draw in \eqref{eqn:mean_covariance} is available whenever $f(\mathbf{x}, \mathbf{y})$ is quadratic in $\mathbf{x}$. This covers a broad class of imaging problems, but it excludes settings with non-Gaussian noise (e.g., Poisson noise in low-light imaging), a nonlinear forward operator (e.g., phase retrieval), or a linear operator $\mathbf{A}$ with no exploitable structure, in which case inverting $\mathbf{Q}^{-1}$ costs $\mathcal{O}(d^3)$ time and storing it costs $\mathcal{O}(d^2)$ memory. We therefore also propose a variant, named \textbf{Joint-RED-KLwSGS}, which applies momentum to both variables. Let $F(\mathbf{x}, \mathbf{z})=f(\mathbf{x}, \mathbf{y})+\beta g(\mathbf{z})+\lVert \mathbf{x}-\mathbf{z}\rVert^2_2/(2\rho^2)$ be the joint potential of $\pi_\rho$. Joint-RED-KLwSGS runs kinetic Langevin diffusions on the full state $(\mathbf{X}, \mathbf{Z})$, with momenta $(\mathbf{U}, \mathbf{V})$
\begin{align}
\label{eqn:joint-RED-KLwSGS}
\begin{split}
    d\mathbf{X}_t &= \mathbf{U}_t dt,\\
    d\mathbf{U}_t &= -\gamma \mathbf{U}_t dt - u \nabla_{\mathbf{X}}F(\mathbf{X}_t,\mathbf{Z}_t)+\sqrt{2\gamma u} d\mathbf{B}^{x}_t,\\
    d\mathbf{Z}_t &= \mathbf{V}_t dt,\\
    d\mathbf{V}_t &= -\gamma \mathbf{V}_t dt - u \nabla_{\mathbf{Z}}F(\mathbf{X}_t,\mathbf{Z}_t)+\sqrt{2\gamma u} d\mathbf{B}^{z}_t,
\end{split}
\end{align}
where $\nabla_{\mathbf{X}}F(\mathbf{x},\mathbf{z})= \nabla_{\mathbf{X}}f(\mathbf{x},\mathbf{y})+\rho^{-2}(\mathbf{x}-\mathbf{z})$, and $\nabla_{\mathbf{Z}}F(\mathbf{x},\mathbf{z})= \beta (\mathbf{z}-D_{\nu}(\mathbf{z}))+\rho^{-2}(\mathbf{z}-\mathbf{x})$. Discretizing \eqref{eqn:joint-RED-KLwSGS} with step size $h$ gives the \textbf{Joint-RED-KLwSGS} sampler
\begin{align*}
    \mathbf{U}^{(k+1)} &= (1-h\gamma)\mathbf{U}^{(k)} - hu \nabla_{\mathbf{X}} F(\mathbf{X}^{(k)}, \mathbf{Z}^{(k)}) \\
    &+ \sqrt{2\gamma hu} \mathbf{N}_x^{(k)},\\
    \mathbf{X}^{(k+1)} &= \mathbf{X}^{(k)} + h \mathbf{U}^{(k+1)},\\
    \mathbf{V}^{(k+1)} &= (1-h\gamma)\mathbf{V}^{(k)} - hu \nabla_{\mathbf{Z}} F(\mathbf{X}^{(k)}, \mathbf{Z}^{(k)}) \\
    &+ \sqrt{2\gamma hu} \mathbf{N}_z^{(k)},\\
    \mathbf{Z}^{(k+1)} &= \mathbf{Z}^{(k)} + h \mathbf{V}^{(k+1)},
\end{align*}
where $(\mathbf{N}_x^{(k)})_{k\in \mathbb{N}}$ and $(\mathbf{N}_z^{(k)})_{k\in \mathbb{N}}$ are independent sequences of standard Gaussian vectors.
We assess this variant empirically in Section \ref{sec:exps}, and its theoretical guarantees are left to future work.

\section{Numerical Experiments}\label{sec:exps}
\subsection{Experimental setup}

We conduct experiments on two datasets, each consisting of $256\times 256$ RGB images ($d=3\times 256^2$): FFHQ \citep{karras2019style} and ImageNet \citep{deng2009imagenet}. All images are normalized to $\left[0,1 \right]$. 
Pretrained diffusion models are taken from \citep{dhariwal2021diffusion,choi2021ilvr} and used without fine-tuning. We use three image inversion tasks to evaluate the performance of compared methods, which share the same forward model $\mathbf{y} = \mathbf{A}(\mathbf{x})+\mathbf{n}$ with $\mathbf{A}$ designed as follows. \textbf{Gaussian deblurring:} $\mathbf{A}(\mathbf{x})=\mathcal{K}_G*\mathbf{x}$, convolution with an isotropic Gaussian kernel of size $61\times 61$ and standard deviation $\sigma_k=3.0$, under circular boundary conditions. \textbf{Motion deblurring:} $\mathbf{A}(\mathbf{x})=\mathcal{K}_M*\mathbf{x}$, convolution with a $61\times 61$ motion kernel of size $61\times 61$ generated from a random walk $\theta_t=\theta_{t-1}+\epsilon_t$, $\epsilon_t\sim \mathcal{N}(0, \alpha^2)$, with intensity $\alpha=0.5$, rasterized into an anisotropic, non-symmetric kernel and applied under circular boundary conditions. \textbf{Super-resolution:} $\mathbf{A}=\mathbf{TB}$, where $\mathbf{T}$ is a Gaussian blur with $9\times 9$ kernel and standard deviation 1.5, and $\mathbf{T}$ performs non-overlapping box downsampling by a factor of 4 in both spatial dimensions, taking a $256\times 256$ blurred image to $64\times 64$. 
For super-resolution, the observation is corrupted by Gaussian noise at an SNR of $40$\,dB. For deblurring, the diagonal entries $\sigma_i^2$ of $\mathbf{\Omega}^{-1}$ are drawn independently, with $\sigma_i=0.157$ with probability $0.35$ and $\sigma_i=0.051$ with probability $0.65$. Details of experiments are given in Appendix~\ref{app:additionalresults}.

\begin{table*}[t]
\centering
\caption{FFHQ $256\times256$ data set: image reconstruction (PSNR, SSIM, LPIPS) obtained by the compared methods. \textbf{Bold}: best, \underline{underline}: second best.}
\label{tab:ffhq_results}
\renewcommand{\arraystretch}{1.3}
\resizebox{\textwidth}{!}{%
\begin{tabular}{cl cccc ccccc}
\toprule
& & SPA & TV-ADMM & PnP-ADMM & DDRM & PnP-SGS & RED-LwSGS & RED-KLwSGS (ours) & Joint-RED-KLwSGS (ours) \\
\midrule
\multirow{3}{*}{\rotatebox{90}{\shortstack{Deblurring\\(Gaussian)}}}
& PSNR $\uparrow$  & 23.17 & 22.37 & 24.93 & 23.36 & 27.96 & 28.01 & \textbf{29.51} & \underline{29.32} \\
& SSIM $\uparrow$  & 0.499 & 0.801 & 0.812 & 0.767 & 0.837 & 0.842 & \textbf{0.865} & \underline{0.853} \\
& LPIPS $\downarrow$ & 0.452 & 0.507 & 0.441 & \underline{0.332} & \textbf{0.331} & 0.352 & 0.336 & 0.341 \\
\midrule
\multirow{3}{*}{\rotatebox{90}{\shortstack{Deblurring\\(motion)}}}
& PSNR $\uparrow$  & 17.73 & 21.36 & 24.65 & N/A & 28.46 & 28.22 & \textbf{29.74} & \underline{29.43} \\
& SSIM $\uparrow$  & 0.211 & 0.751 & 0.825 & N/A & 0.828 & 0.823 & \textbf{0.858} & \underline{0.851} \\
& LPIPS $\downarrow$ & 0.446 & 0.508 & 0.405 & N/A & \underline{0.294} & \textbf{0.292} & 0.302 & 0.315 \\
\midrule
\multirow{3}{*}{\rotatebox{90}{\shortstack{Superres.\\($\times4$)}}}
& PSNR $\uparrow$  & N/A & 23.86 & \underline{26.55} & 25.36 & 25.99 & 25.61 & 26.35 & \textbf{26.73} \\
& SSIM $\uparrow$  & N/A & 0.803 & \textbf{0.865} & \underline{0.835} & 0.812 & 0.831 & 0.823 & 0.815 \\
& LPIPS $\downarrow$ & N/A & 0.428 & 0.353 & 0.294 & \textbf{0.279} & \underline{0.286} & 0.291 & 0.297 \\
\bottomrule
\end{tabular}%
}
\end{table*}

\begin{table*}[t]
\centering
\caption{ImageNet $256\times256$ data set: image reconstruction (PSNR, SSIM, LPIPS) obtained by the compared methods. 
}
\label{tab:imagenet_results}
\renewcommand{\arraystretch}{1.3}
\resizebox{\textwidth}{!}{%
\begin{tabular}{cl cccc ccccc}
\toprule
& & SPA & TV-ADMM & PnP-ADMM & DDRM & PnP-SGS & RED-LwSGS & RED-KLwSGS (ours) & Joint-RED-KLwSGS (ours) \\
\midrule
\multirow{3}{*}{\rotatebox{90}{\shortstack{Deblurring\\(Gaussian)}}}
& PSNR $\uparrow$  & 21.08 & 19.99 & 21.81 & 22.73 & 21.76 & 23.12 & \textbf{23.92} & \underline{23.85} \\
& SSIM $\uparrow$  & 0.577 & 0.634 & 0.669 & 0.705 & 0.701 & 0.712 & \textbf{0.735} & \underline{0.729} \\
& LPIPS $\downarrow$ & 0.537 & 0.588 & 0.519 & 0.427 & 0.399 & \underline{0.359} & \textbf{0.349} & 0.365 \\
\midrule
\multirow{3}{*}{\rotatebox{90}{\shortstack{Deblurring\\(motion)}}}
& PSNR $\uparrow$  & 20.49 & 20.79 & 21.98 & N/A & 21.47 & 22.73 & \textbf{24.52} & \underline{24.21} \\
& SSIM $\uparrow$  & 0.681 & 0.677 & 0.702 & N/A & 0.695 & 0.709 & \textbf{0.743} & \underline{0.739} \\
& LPIPS $\downarrow$ & 0.538 & 0.525 & 0.483 & N/A & 0.372 & \underline{0.339} & 0.342 & \textbf{0.325} \\
\midrule
\multirow{3}{*}{\rotatebox{90}{\shortstack{Superres.\\($\times4$)}}}
& PSNR $\uparrow$  & N/A & 22.17 & 23.75 & 24.96 & 24.33 & 23.89 & \textbf{25.87} & \underline{25.23} \\
& SSIM $\uparrow$  & N/A & 0.679 & 0.761 & \underline{0.790} & 0.772 & 0.745 & \textbf{0.812} & 0.779 \\
& LPIPS $\downarrow$ & N/A & 0.523 & 0.433 & \underline{0.339} & 0.418 & \underline{0.339} & 0.343 & \textbf{0.325} \\
\bottomrule
\end{tabular}%
}
\end{table*}

\subsection{Acceleration of RED-KLwSGS and Joint-RED-KLwSGS}
We first verify the acceleration effect of RED-KLwSGS and Joint-RED-KLwSGS over their nonaccelerated counterpart, RED-LwSGS.  We set $\gamma=2$ for RED-KLwSGS and Joint-RED-KLwSGS, and the same step size $h=0.001$ for all methods. We first qualitatively evaluate methods through visual inspection. We select one FFHQ image as the clean image $\mathbf{x}$ for the Gaussian deblurring task and apply the three methods to recover $\mathbf{x}$ from its noisy observation. Figure \ref{fig:convergence_motion} visualizes $\mathbf{X}$ over iterations. RED-LwSGS remains visibly noisy thorugh $k=300-400$, while both accelerated variants reach comparable, good visual quality by $k\approx200$. All three converge to visually similar best reconstructions by $k=800$.
We further evaluate methods quantitatively using three widely-used metrics: peak signal-to-noise-ratio (PSNR), structural similarity index (SSIM), and learned perceptual image patch similarity (LPIPS). We report the results averaged over ten images for all three tasks. Figure \ref{fig:convergence_all_tasks} shows the quality of $\mathbf{X}$ over 800 iterations for the three methods. RED-KLwSGS and Joint-RED-KLwSGS converge faster than the non-accelerated RED-LwSGS. For Gaussian deblurring, RED-KLwSGS and Joint-RED-KLwSGS reach high-quality reconstructions by iteration 400 (PSNR $\approx$ 27dB, SSIM $\approx$ 0.67, LPIPS $\approx$ 0.32), while RED-LwSGS reaches substantially lower quality at the same iteration (PSNR $\approx$ 17dB, SSIM $\approx$ 0.15, LPIPS $\approx$ 0.63). By iteration 500, RED-KLwSGS and Joint-RED-KLwSGS reach their highest quality, while RED-LwSGS needs more than 600 iterations to reach comparable quality. We observe a similar acceleration effect can be found on the other two tasks, and on ImageNet (see Appendix \ref{app:additionalresults}). These results confirm the acceleration effect of RED-KLwSGS and Joint-RED-KLwSGS.

\subsection{Comparison with existing methods}
We compare our methods to existing approaches in terms of the reconstruction quality, including SPA \citep{vono2019split}, TV-ADMM, PnP-ADMM \citep{chan2016plug}, DDRM \citep{kawar2022denoising}, and PnP-SGS \citep{coeurdoux2024plug}. Note that PnP-ADMM, TV-ADMM, DDRM yield point estimates only, while PnP-SGS, RED-LwSGS, and our methods provide comprehensive description of the targeted posterior distribution, allowing uncertainty quantification. We compare methods on PSNR, SSIM, and LPIPS, averaged over 1000 test images. Results for RED-LwSGS, and our methods are obtained from our own experiments, while results for the remaining baselines are taken from \citep{coeurdoux2024plug}. Table \ref{tab:ffhq_results} and Table \ref{tab:imagenet_results} report PSNR, SSIM, and LPIPS on FFHQ and ImageNet across three tasks. RED-KLwSGS achieves the best PSNR and SSIM in five of six task-rows on both datasets, improving over RED-LwSGS by up to $1.8$ dB in PSNR at identical per-iteration cost. The gain comes entirely from the momentum term in the $\mathbf{z}$-update (see Table \ref{tab:pnp_vs_lwsgs_full}). Joint-RED-KLwSGS ranks second on PSNR/SSIM in nearly every row, but attains the best LPIPS in three of six rows, consistent with a perception-distortion trade-off between two variants.

\section{Conclusion}
We proposed RED-KLwSGS, a kinetic Langevin sampler for split Gibbs sampling under RED regularization framework, which pairs the exact Gaussian update for the data variable with a kinetic Langevin update for the auxiliary variable. We establish non-asymptotic Wasserstein-2 convergence guarantees for both the continuous-time diffusion and its discretization. Experiments on FFHQ and ImageNet across three inverse problems show that RED-KLwSGS consistently improves PSNR and SSIM over the non-accelerated RED-LwSGS at identical per-iteration costs, converging in substantially fewer iterations. We also introduced Joint-RED-KLwSGS, which applies momentum to both variables. It performs competitively on perceptual metrics. Fugure work includes extending the convergence analysis to the joint variant, relaxing the RED conditions, and exploring forward operators beyond quadratic-Gaussian setting considered here.


\bibliographystyle{plainnat}
\bibliography{aistats}

\section*{Checklist}

 \begin{enumerate}

 \item For all models and algorithms presented, check if you include:
 \begin{enumerate}
   \item A clear description of the mathematical setting, assumptions, algorithm, and/or model. [Yes]
   \item An analysis of the properties and complexity (time, space, sample size) of any algorithm. [Yes]
   \item (Optional) Anonymized source code, with specification of all dependencies, including external libraries. [Not Applicable]
 \end{enumerate}

 \item For any theoretical claim, check if you include:
 \begin{enumerate}
   \item Statements of the full set of assumptions of all theoretical results. [Yes]
   \item Complete proofs of all theoretical results. [Yes]
   \item Clear explanations of any assumptions. [Yes]     
 \end{enumerate}

 \item For all figures and tables that present empirical results, check if you include:
 \begin{enumerate}
   \item The code, data, and instructions needed to reproduce the main experimental results (either in the supplemental material or as a URL). [Yes]
   \item All the training details (e.g., data splits, hyperparameters, how they were chosen). [Yes]
         \item A clear definition of the specific measure or statistics and error bars (e.g., with respect to the random seed after running experiments multiple times). [Yes]
         \item A description of the computing infrastructure used. (e.g., type of GPUs, internal cluster, or cloud provider). [Yes]
 \end{enumerate}

 \item If you are using existing assets (e.g., code, data, models) or curating/releasing new assets, check if you include:
 \begin{enumerate}
   \item Citations of the creator If your work uses existing assets. [[Not Applicable]]
   \item The license information of the assets, if applicable. [Not Applicable]
   \item New assets either in the supplemental material or as a URL, if applicable. [Not Applicable]
   \item Information about consent from data providers/curators. [Not Applicable]
   \item Discussion of sensible content if applicable, e.g., personally identifiable information or offensive content. [Not Applicable]
 \end{enumerate}

 \item If you used crowdsourcing or conducted research with human subjects, check if you include:
 \begin{enumerate}
   \item The full text of instructions given to participants and screenshots. [Not Applicable]
   \item Descriptions of potential participant risks, with links to Institutional Review Board (IRB) approvals if applicable. [Not Applicable]
   \item The estimated hourly wage paid to participants and the total amount spent on participant compensation. [Not Applicable]
 \end{enumerate}

 \end{enumerate}

\onecolumn
\appendix
\date{}
\title{\textbf{Appendix for Kinetic Langevin Meets Split Gibbs: Accelerated Posterior Sampling for Imaging Inverse Problems with Diffusion Priors}}
\maketitle
\section{Diffusion models and Split Gibbs Samplers}
\label{app:diffusionmodels}
\subsection{Diffusion models as stochastic denoisers}
Denoising diffusion models and score-based models are popular classes of generative models. A denoising diffusion probabilistic model (DDPM) \citep{ho2020denoising,kingma2021variational} make use of two Markov chains: a forward chain that converts data to pure noise, and backward chain that converts noise back to data.\\

\paragraph{Forward process.}
A DDPM learns the prior of clean images by gradually corrupting them with Gaussian noise. Fix a schedule $\alpha_t\in(0,1)$ and let
$\bar\alpha_t=\prod_{s\le t}\alpha_s$. The forward chain is
$q(\mathbf{u}_t\mid\mathbf{u}_{t-1})=\mathcal{N}\big(\sqrt{\alpha_t}\,\mathbf{u}_{t-1},\,(1-\alpha_t)\mathbf{I}_d\big)$,
$\mathbf{u}_0=\mathbf{z}$, and composing the steps gives the closed form
\begin{equation}
\mathbf{u}_t=\sqrt{\bar\alpha_t}\,\mathbf{z}+\sqrt{1-\bar\alpha_t}\,\boldsymbol{\epsilon},
\qquad \boldsymbol{\epsilon}\sim\mathcal{N}(\mathbf{0},\mathbf{I}_d).
\label{eqn:ddpm_forward}
\end{equation}

\paragraph{Training.}
A network $\boldsymbol{\epsilon}_\theta(\mathbf{u}_t,t)$ is trained to predict the noise that was added,
\begin{equation*}
\min_\theta\ \mathbb{E}_{\mathbf{z},t,\boldsymbol{\epsilon}}\big\|\boldsymbol{\epsilon}-\boldsymbol{\epsilon}_\theta(\mathbf{u}_t,t)\big\|_2^2 .
\end{equation*}
The minimizer is the conditional mean, $\boldsymbol{\epsilon}_\theta(\mathbf{u}_t,t)=\mathbb{E}[\boldsymbol{\epsilon}\mid\mathbf{u}_t]$.

\paragraph{DDPM is a score and a denoiser.}
From \eqref{eqn:ddpm_forward}, $\nabla_{\mathbf{u}_t}\log q(\mathbf{u}_t\mid\mathbf{z})=-\boldsymbol{\epsilon}/\sqrt{1-\bar\alpha_t}$.
Averaging over $\mathbf{z}\mid\mathbf{u}_t$ gives the score of the noised prior $p_t$,
\begin{equation}
\nabla\log p_t(\mathbf{u}_t)=-\frac{\mathbb{E}[\boldsymbol{\epsilon}\mid\mathbf{u}_t]}{\sqrt{1-\bar\alpha_t}}
\approx-\frac{\boldsymbol{\epsilon}_\theta(\mathbf{u}_t,t)}{\sqrt{1-\bar\alpha_t}}.
\label{eqn:ddpm_score}
\end{equation}
Tweedie's formula \citep{efron2011tweedie} for the Gaussian channel \eqref{eqn:ddpm_forward} states
$\mathbb{E}[\sqrt{\bar\alpha_t}\,\mathbf{z}\mid\mathbf{u}_t]=\mathbf{u}_t+(1-\bar\alpha_t)\nabla\log p_t(\mathbf{u}_t)$.
Substituting \eqref{eqn:ddpm_score},
\begin{equation}
\mathbb{E}[\mathbf{z}\mid\mathbf{u}_t]=\frac{\mathbf{u}_t-\sqrt{1-\bar\alpha_t}\,\boldsymbol{\epsilon}_\theta(\mathbf{u}_t,t)}{\sqrt{\bar\alpha_t}}
=:\hat{\mathbf{z}}_\theta(\mathbf{u}_t,t).
\label{eqn:ddpm_tweedie}
\end{equation}
So \emph{one network evaluation returns both a score and a posterior-mean estimate of the clean image}
(the MMSE denoiser).

\paragraph{Backward process.}
To obtain an actual \emph{sample} one runs the learned reverse chain
\begin{equation}
\mathbf{u}_{t-1}=\frac{1}{\sqrt{\alpha_t}}\Big(\mathbf{u}_t-\frac{1-\alpha_t}{\sqrt{1-\bar\alpha_t}}\,
\boldsymbol{\epsilon}_\theta(\mathbf{u}_t,t)\Big)+\sqrt{\tilde\beta_t}\,\mathbf{w}_t,
\qquad \mathbf{w}_t\sim\mathcal{N}(\mathbf{0},\mathbf{I}_d),
\label{eqn:ddpm_reverse}
\end{equation}
which takes one network call per step. Each step must be small for its Gaussian transition to be accurate,
so a good sample needs many steps.

\subsection{Connection between diffusion models and Split Gibbs Samplers}

The conditional \eqref{eqn:conditional_z} is the posterior of a denoising problem: recover $\mathbf{z}$ from
$\mathbf{x}=\mathbf{z}+\rho\mathbf{n}$, $\mathbf{n}\sim\mathcal{N}(\mathbf{0},\mathbf{I}_d)$, under the prior $e^{-\beta g}$.
Dividing \eqref{eqn:ddpm_forward} by $\sqrt{\bar\alpha_t}$ gives $\mathbf{z}+\sigma_t\boldsymbol{\epsilon}$ with
$\sigma_t^2=(1-\bar\alpha_t)/\bar\alpha_t$. So noise level $\rho$ corresponds to the timestep $t^\star$ with
$\sigma_{t^\star}\approx\rho$, and $\mathbf{x}$ plays the role of $\mathbf{u}_{t^\star}/\sqrt{\bar\alpha_{t^\star}}$.
For $\beta=1$ and a diffusion prior, the conditional is exactly the denoising posterior this model was trained for.
Both samplers below use the same pretrained network; they differ in how they use it.

\paragraph{PnP-SGS \citep{coeurdoux2024plug}: use the network as a sampler.}
Each $\mathbf{Z}$-update draws an approximate sample of $p(\mathbf{z}\mid\mathbf{x})$. It starts the reverse chain at
$\mathbf{u}_{t^\star}=\sqrt{\bar\alpha_{t^\star}}\,\mathbf{x}$ and runs \eqref{eqn:ddpm_reverse} from $t^\star$ down to $0$.
The cost is about $t^\star$ network evaluations per Gibbs iteration, and $t^\star$ grows with $\rho$.

\paragraph{RED-LwSGS \citep{faye2024regularization}: use the network as a score.}
It does not sample from the network. It evaluates the network once at the current iterate $\mathbf{Z}^{(k)}$, at a fixed
noise level $\nu$ (the denoising strength of $D_\nu$), and uses \eqref{eqn:ddpm_tweedie} to get
$D_\nu(\mathbf{Z}^{(k)})\approx\hat{\mathbf{z}}_\theta$. The prior part of the score of \eqref{eqn:conditional_z} is
$-\beta(\mathbf{z}-D_\nu(\mathbf{z}))$, and one Langevin step is
\begin{equation}
\mathbf{Z}^{(k+1)}=\mathbf{Z}^{(k)}+h\Big[-\beta\big(\mathbf{Z}^{(k)}-D_\nu(\mathbf{Z}^{(k)})\big)
-\frac{1}{\rho^2}\big(\mathbf{Z}^{(k)}-\mathbf{X}^{(k)}\big)\Big]+\sqrt{2h}\,\mathbf{N}^{(k)}.
\label{eqn:red_langevin_step}
\end{equation}
By Tweedie, the residual $\mathbf{z}-D_\nu(\mathbf{z})$ equals $-\nu^2$ times the score of the prior smoothed at noise level
$\nu$. So for $\beta=1/\nu^2$ the update \eqref{eqn:red_langevin_step} is a Langevin step on a smoothed diffusion prior,
and for other $\beta$ the correspondence is only approximate. The randomness comes from $\sqrt{2h}\,\mathbf{N}^{(k)}$, and the
gradual commitment happens across Gibbs iterations rather than inside one. The cost is one network evaluation per iteration.

\section{Proof of Theorem \ref{theorem1}}
\label{prooftheorem1}
\begin{proof}
   Our proof relies on a synchronous coupling argument. We consider two stochastic processes, $(\mathbf{X}_t, \mathbf{Z}_t, \mathbf{V}_t)\sim \mathbf{\Phi}_t\mu$ and $(\widetilde{\mathbf{X}}_t, \widetilde{\mathbf{Z}}_t, \widetilde{\mathbf{V}}_t)\sim \mathbf{\Phi}_t\widetilde{\mu}$, coupled through independent draws from $\mu$ and $\widetilde{\mu}$ at $t=0$, and through a shared Gaussian noise $\epsilon_t$ and shared Brownian motion $\mathbf{B}_t$ thereafter:
   \begin{align}
   \label{eqn:dynamic_Z}
   \begin{split}
       \mathbf{X}_t &= M(\mathbf{Z}_t) + \mathbf{Q}^{-\frac{1}{2}}\boldsymbol{\epsilon}_t,
       d\mathbf{Z}_t = \mathbf{V}_t dt,\\
       d\mathbf{V}_t &= -\gamma \mathbf{V}_t dt - u \left[\beta \nabla g(\mathbf{Z}_t) + \rho^{-2}\left(\mathbf{Z}_t - \mathbf{X}_t \right) \right]dt + \sqrt{2\gamma u} d\mathbf{B}_t,
    \end{split}
    \end{align}
    and identically for $(\widetilde{\mathbf{X}}_t, \widetilde{\mathbf{Z}}_t, \widetilde{\mathbf{V}}_t)$, driven by the same $\boldsymbol{\epsilon}_t, \mathbf{B}_t$:
    \begin{align}
    \label{eqn:dynamic_Z_tidle}
   \begin{split}
       \widetilde{\mathbf{X}}_t &= M(\widetilde{\mathbf{Z}}_t) + \mathbf{Q}^{-\frac{1}{2}}\epsilon_t,
       d\widetilde{\mathbf{Z}}_t = \widetilde{\mathbf{V}}_t dt,\\
       d\widetilde{\mathbf{V}}_t &= -\gamma \widetilde{\mathbf{V}}_t dt - u \left[\beta \nabla g(\widetilde{\mathbf{Z}}_t) + \rho^{-2}\left(\widetilde{\mathbf{Z}}_t - \widetilde{\mathbf{X}}_t \right) \right]dt + \sqrt{2\gamma u} d\mathbf{B}_t
    \end{split}
   \end{align}
   Write $\mathbf{\Xi}_t=\left(\mathbf{X}_t, \mathbf{Z}_t, \mathbf{V}_t \right)$, $\mathbf{\Xi}_t=\left(\widetilde{\mathbf{X}}_t, \widetilde{\mathbf{Z}}_t, \widetilde{\mathbf{V}}_t \right)$. By the definition of the Wasserstein distance, 
   \begin{align*}
   \begin{split}
       \mathcal{W}^2_2\left(\mathbf{\Phi}_t\mu, \mathbf{\Phi}_t\widetilde{\mu} \right) &\leq \mathbb{E}\left[ \left\lVert \mathbf{\Xi}_t-\widetilde{\mathbf{\Xi}}_t\right\rVert^2_2\right]=\mathbb{E}\left[\left\lVert \rho^{-2}\mathbf{Q}^{-1}\left(\mathbf{Z}_t-\widetilde{\mathbf{Z}}_t \right)\right\rVert^2_2\right]+\mathbb{E}\left[\left\lVert \mathbf{Z}_t-\widetilde{\mathbf{Z}}_t \right\rVert^2_2\right]+\mathbb{E}\left[\left\lVert \mathbf{V}_t-\widetilde{\mathbf{V}}_t \right\rVert^2_2\right],
\end{split}
\end{align*}
using $\mathbf{X}_t-\widetilde{\mathbf{X}}_t = \rho^{-2}\mathbf{Q}^{-1}\left(\mathbf{Z}_t-\widetilde{\mathbf{Z}}_t \right)$ (the shared $\boldsymbol{\epsilon}_t$ cancels in this difference). Bounding the first term in the operator norm gives
\begin{align}
\label{eqn:2}
\begin{split}
       \mathcal{W}^2_2\left(\mathbf{\Phi}_t\mu, \mathbf{\Phi}_t\widetilde{\mu} \right)
       &\leq \left(1 + \rho^{-4}\lVert\mathbf{Q}^{-1} \rVert^2_{\mathrm{op}} \right)\mathbb{E}\left[\left\lVert \mathbf{Z}_t-\widetilde{\mathbf{Z}}_t \right\rVert^2_2+\left\lVert \mathbf{V}_t-\widetilde{\mathbf{V}}_t \right\rVert^2_2\right].
    \end{split}
   \end{align}
   Thus, we aim to bound the coupling difference $\mathbb{E}\left[\left\lVert \mathbf{Z}_t-\widetilde{\mathbf{Z}}_t \right\rVert^2_2+\left\lVert \mathbf{V}_t-\widetilde{\mathbf{V}}_t \right\rVert^2_2\right]$. Write $\Delta\mathbf{Z}_t = \mathbf{Z}_t-\widetilde{\mathbf{Z}}_t$ and $\Delta\mathbf{V}_t = \mathbf{V}_t-\widetilde{\mathbf{V}}_t$. Since both processes share $\mathbf{B}_t$, the noise chancels in the difference and it follows that
   \begin{align*}
   \begin{split}
       \frac{d}{dt}\left(\Delta\mathbf{Z}_t + \Delta\mathbf{V}_t   \right)&=\left(1-\gamma \right)\Delta \mathbf{V}_t- u \left[ \beta \nabla g(\mathbf{Z}_t)-\beta \nabla g(\widetilde{\mathbf{Z}}_t)+\rho^{-2}\left( \mathbf{Z}_t - \widetilde{\mathbf{Z}}_t\right)-\rho^{-2}\left( \mathbf{X}_t - \widetilde{\mathbf{X}}_t\right)\right]
   \end{split}
   \end{align*}

   Using Taylor's theorem ($g$ twice differentiable), and substituting $\mathbf{X}_t-\widetilde{\mathbf{X}}_t = \rho^{-2}\mathbf{Q}^{-1}\left(\mathbf{Z}_t-\widetilde{\mathbf{Z}}_t \right)$:
   \begin{equation}
   \label{eqn:DeltaZ_DeltaV}
        \frac{d}{dt}\left(\Delta\mathbf{Z}_t + \Delta\mathbf{V}_t   \right) =\left(1-\gamma \right)\Delta \mathbf{V}_t - u \mathbf{H}_t\Delta\mathbf{Z}_t= -\left[(\gamma - 1)\Delta\mathbf{V}_t + u \mathbf{H}_t \Delta \mathbf{Z}_t\right],
   \end{equation}

   where \begin{equation*}
       \mathbf{H}_t = \beta\left(\int_{0}^{1}\nabla^2 g(\mathbf{Z}_t + w(\widetilde{\mathbf{Z}}_t-\mathbf{Z}_t))dw\right)+\rho^{-2}\mathbf{I}_d-\rho^{-4}\mathbf{Q}^{-1}.
   \end{equation*}
   
   We next bound $\mathbf{H}_t$. Following \cite{faye2024regularization}, define $\widetilde{\mathbf{A}}=\begin{bmatrix}
    \frac{1}{\sigma}\mathbf{A}^\top & \frac{1}{\rho}\mathbf{I}_d
\end{bmatrix}$, so $\mathbf{Q}=\widetilde{\mathbf{A}}^\top\widetilde{\mathbf{A}}$. We define the orthogonal projector $\mathbf{P}=\widetilde{\mathbf{A}}\mathbf{Q}^{-1}\widetilde{\mathbf{A}}^\top$, so $\mathbf{0} \preceq \mathbf{P} \preceq \mathbf{I}$. With $\mathbf{J}=\begin{bmatrix}
    \mathbf{0} & \frac{1}{\rho}\mathbf{I}_d
\end{bmatrix}^\top$, we have $\mathbf{J}^\top \mathbf{P} \mathbf{J}=\rho^{-4} \mathbf{Q}^{-1}$. Combining with $\beta m_g \mathbf{I}_d\preceq \beta \nabla^2 g\preceq \left(\beta M_g +\frac{1}{\rho^2} \right)\mathbf{I}_d$, this gives, 
$\beta m_g\mathbf{I}_d \preceq \mathbf{H}_t \preceq (\beta M_g +\rho^{-2})\mathbf{I}_d$.

Next we apply the techniques of  \citep{cheng2018underdamped} to prove the contraction, in four steps.\\
\noindent
\textbf{Step 1:} Collecting \eqref{eqn:dynamic_Z}, \eqref{eqn:dynamic_Z_tidle}, and \eqref{eqn:DeltaZ_DeltaV} to yield
\begin{equation*}
    \frac{d}{dt}\left(\Delta \mathbf{Z}_t \right)= \Delta \mathbf{V}_t,  \frac{d}{dt}\left(\Delta\mathbf{Z}_t + \Delta\mathbf{V}_t   \right)= -\left[(\gamma - 1)\Delta\mathbf{V}_t + u \mathbf{H}_t \Delta \mathbf{Z}_t\right].
\end{equation*}

\noindent
\textbf{Step 2:} Differentiating the Lyapunov functional using the two derivatives above:
\begin{align*}
    \begin{split}
        \frac{d}{dt}\left[\lVert \Delta \mathbf{Z}_t+\Delta \mathbf{V}_t \rVert^2_2 + \lVert\Delta \mathbf{Z}_t \rVert^2_2 \right] &= -2 \langle\left(\Delta \mathbf{Z}_t + \Delta \mathbf{V}_t, \Delta \mathbf{Z}_t \right), \left((\gamma -1)\Delta \mathbf{V}_t+u \mathbf{H}_t \Delta \mathbf{Z}_t \right), -\Delta \mathbf{V}_t \rangle\\
        & = -2 \begin{bmatrix}
           \Delta \mathbf{Z}_t^\top + \Delta \mathbf{V}_t^\top & \Delta \mathbf{Z}_t^\top 
        \end{bmatrix}
        \underbrace{\begin{bmatrix}
           (\gamma - 1)\mathbf{I} & u \mathbf{H}_t - (\gamma - 1)\mathbf{I}\\
           -\mathbf{I} & \mathbf{I}
        \end{bmatrix}}_{\mathbf{S}_t}
        \begin{bmatrix}
           \Delta \mathbf{Z}_t + \Delta \mathbf{V}_t \\ \Delta \mathbf{Z}_t 
        \end{bmatrix}\\
        & = -2 \begin{bmatrix}
           \Delta \mathbf{Z}_t^\top + \Delta \mathbf{V}_t^\top & \Delta \mathbf{Z}_t^\top 
        \end{bmatrix}
        \mathbf{R}_t
        \begin{bmatrix}
           \Delta \mathbf{Z}_t + \Delta \mathbf{V}_t \\ \Delta \mathbf{Z}_t 
        \end{bmatrix},
    \end{split}
\end{align*}
where $\mathbf{R}_t = \frac{1}{2} \left(\mathbf{S}_t + \mathbf{S}_t^\top \right)$ (only the symmetric part of $\mathbf{S}_t$ contributes to the quadratic form).

\noindent
\textbf{Step 3:} The eigenvalues of $\mathbf{R}_t$ solve 
\begin{equation*}
    (\gamma - 1 - \lambda) (1-\lambda) - \frac{1}{4}(u \Lambda_j - \gamma)^2=0,
\end{equation*}
where $\Lambda_j$ ($j\in\left\{ 1,2,...,d\right\}$) are the eigenvalues of $\mathbf{H}_t$. Setting $u=1/(\beta M_g + \rho^{-2})$ and $\gamma=2$:
\begin{equation*}
\lambda_j^{*}=1\pm\left(1 - \frac{\Lambda_j}{2\left(\beta M_g + \rho^{-2} \right)} \right), \text{ for } j\in\left\{ 1,2,...,d\right\}, 
\end{equation*}
which, using $\Lambda_j \leq \beta M_g + \rho^{-2}$ gives $\lambda_{\mathrm{min}}(\mathbf{R}_t)\geq \frac{1}{2\kappa}$.

\noindent
\textbf{Step 4}: Combining Steps 2-3 gives
\begin{equation*}
    \frac{d}{dt}\left[\lVert \Delta \mathbf{Z}_t+\Delta \mathbf{V}_t \rVert^2_2 + \lVert\Delta \mathbf{Z}_t \rVert^2_2 \right]  \leq - \frac{1}{\kappa}\left[\lVert \Delta \mathbf{Z}_t+\Delta \mathbf{V}_t \rVert^2_2 + \lVert\Delta \mathbf{Z}_t \rVert^2_2 \right].
\end{equation*}
Gronwall's inequality then gives
\begin{equation}
\label{enq:Gronwall}
    \lVert \Delta \mathbf{Z}_t+\Delta \mathbf{V}_t \rVert^2_2 + \lVert\Delta \mathbf{Z}_t \rVert^2_2  \leq e^{-\frac{t}{\kappa}} \left[\lVert \Delta \mathbf{Z}_0+\Delta \mathbf{V}_0 \rVert^2_2 + \lVert\Delta \mathbf{Z}_0 \rVert^2_2 \right].
\end{equation}

Let $\mathbf{M} = \begin{pmatrix}2&1\\1&1\end{pmatrix}$, so
\begin{equation*}
    \lVert \Delta \mathbf{Z}_t+\Delta \mathbf{V}_t \rVert^2_2 + \lVert\Delta \mathbf{Z}_t \rVert^2_2=            \begin{bmatrix}
           \Delta \mathbf{Z}_t^\top  & \Delta \mathbf{V}_t^\top 
        \end{bmatrix}
        (\mathbf{M}\otimes \mathbf{I}_d)
        \begin{bmatrix}
           \Delta \mathbf{Z}_t  \\  \Delta \mathbf{V}_t 
        \end{bmatrix}.
\end{equation*}   
Since the eigenvalues of a Kronecker product $\mathbf{M}\otimes I_d$ are those of $\mathbf{M}$ (each repeated $d$ times, independent of $d$), and $\mathbf{M}$'s characteristic polynomial  $\lambda^2-3\lambda+1=0$ gives $\lambda_{\min}(\mathbf{M}) = \frac{3-\sqrt5}{2}, \lambda_{\max}(\mathbf{M}) = \frac{3+\sqrt5}{2}$. The Rayleigh-quotient bound applied to both sides of  \eqref{enq:Gronwall} gives:
\begin{align}
\label{eqn:1}
\begin{split}
    \lVert \Delta \mathbf{Z}_t \rVert^2_2 + \lVert\Delta \mathbf{V}_t \rVert^2_2 \leq \left(\frac{3+\sqrt{5}}{3-\sqrt{5}} \right)e^{-\frac{t}{\kappa}} \left[\lVert \Delta \mathbf{Z}_0 \rVert^2_2 + \lVert\Delta \mathbf{V}_0 \rVert^2_2 \right] &\leq 7e^{-\frac{t}{\kappa}} \left[\lVert \Delta \mathbf{Z}_0 \rVert^2_2 + \lVert\Delta \mathbf{V}_0 \rVert^2_2 \right]\\
    &\leq 7e^{-\frac{t}{\kappa}} \left[\lVert \Delta \mathbf{X}_0 \rVert^2_2 +\lVert \Delta \mathbf{Z}_0 \rVert^2_2 + \lVert\Delta \mathbf{V}_0 \rVert^2_2 \right].
\end{split}
\end{align}
Choosing the coupling of $(\mathbf{X}_0, \mathbf{Z}_0, \mathbf{V}_0)\sim \mu$ and $(\widetilde{\mathbf{X}}_0, \widetilde{\mathbf{Z}}_0, \widetilde{\mathbf{V}}_0)\sim \widetilde{\mu}$ to be optimal for $\mathcal{W}_2$, so that 
$\mathcal{W}^2_2(\mu, \widetilde{\mu})=\mathbb{E}\left[\lVert \Delta \mathbf{X}_0 \rVert^2+\lVert \Delta \mathbf{Z}_0 \rVert^2 + \lVert\Delta \mathbf{V}_0 \rVert^2 \right]$. Taking expectations in \eqref{eqn:1} and combing with \eqref{eqn:2} gives:
\begin{equation*}
    \mathcal{W}^2_2\left(\mathbf{\Phi}_t\mu, \mathbf{\Phi}_t\widetilde{\mu} \right)\leq C e^{-\frac{t}{\kappa}}\mathcal{W}^2_2\left(\mu, \widetilde{\mu} \right),
\end{equation*}
\end{proof}

\section{Proof of Theorem \ref{theorem2}}
\begin{lemma}
\label{lemma1}
    For a real symmetric $2\times 2$ matrix $\mathbf{N}$, $\mathbf{N}\succeq 0$ if and only if $\mathrm{tr}(\mathbf{N})\geq 0$ and $\mathrm{det}(\mathbf{N})\geq 0$.
\end{lemma}
\begin{proof}
    $\mathbf{N}$ has two eigenvalues $\lambda_1$ and $\lambda_2$. If $\mathrm{det}(\mathbf{N})=\lambda_1\lambda_2\geq 0$, the eigenvalues share sign (or one is zero). If additionally $\mathrm{tr}(\mathbf{N})=\lambda_1 +\lambda_2\geq 0$, they cannot both be negative. Hence both are $\geq 0$. The converse is immediate.
\end{proof}
\begin{lemma}
\label{lemma2}
    Define the following $2\times 2$ matrix:
    \begin{equation*}
        \mathbf{P}=\begin{bmatrix}
            1-h - h^2 q & h - hq + h^2 q\\
            h & 1-h
        \end{bmatrix},
    \end{equation*}
    where $\kappa\geq 1$, $h\in (0,1/2]$, and for $q\in [1/\kappa, 1]$. Then,
    \begin{equation}
\lambda_{\mathrm{max}}\left(\mathbf{P}^\top\mathbf{P} \right) \leq 1 - \frac{h}{\kappa}+\frac{h^2}{\kappa^2}.
    \end{equation}
\end{lemma}
\begin{proof}
    Let $\mathbf{M}=\mathbf{P}^\top\mathbf{P}$, and $c=1 - \frac{h}{\kappa}+\frac{h^2}{\kappa^2}$. We must show that $\mathbf{M}\preceq c\mathbf{I}$, i.e., $\mathbf{N}=c\mathbf{I}-\mathbf{M}\succeq 0$. From Lemma \ref{lemma1}, it suffices to show that $\mathrm{tr}(\mathbf{N})\geq 0$ and $\mathrm{det}(\mathbf{N})\geq 0$.
    Using the identity $\mathrm{det}(c\mathbf{I}-\mathbf{M})=c^2 - c \mathrm{tr}(\mathbf{M})+\mathrm{det}(\mathbf{M})$, we have:
    \begin{equation*}
        \mathrm{tr}(\mathbf{N})= 2 c - T, \mathrm{det}(\mathbf{N})= c^2 - c T + D, 
    \end{equation*}
    where $T = \mathrm{tr}(\mathbf{M}), D= \mathrm{det}(\mathbf{M})$. Direct expansion of $\mathbf{P}^\top \mathbf{P}$ gives the exact polynomials:
    \begin{equation*}
        T=T(h,q) = 2 - 4h + (2-q)^2 h^2 + 2q(2-q)h^3 +2q^2 h^4, D=D(h)=(1-2h)^2.
    \end{equation*}
    \noindent
    \textbf{Step 1}. We prove that for $h\in(0,1/2]$, $T(h, q)$ is strictly decreasing in $q\in [0,1]$.

    We have $\partial T/\partial q = 2h^2 \left[q (2h^2 - 2h + 1)+2h - 2\right]$. 
    Since $2h^2 - 2h + 1> 0$ for all $h$ and $q\leq 1$, $q (2h^2 - 2h + 1)+2h - 2 \leq  (2h^2 - 2h + 1)+2h - 2=2h^2 - 1< 0$ (as $h\leq 1/2$). Hence $\partial T/ \partial q<0$. This gives $\partial (\mathrm{tr}(\mathbf{N}))/\partial q = -\partial T/ \partial q > 0$, and $\partial (\mathrm{det}(\mathbf{N}))/\partial q = -c\partial T/ \partial q > 0$ (as $c>0$). So both quantities are increasing in $q$: the worse case over $q\in[1/\kappa, 1]$ occurs at $q=q_0=1/\kappa\in(0, 1]$. \\

    \noindent
    \textbf{Step 2}. We represent $\mathrm{tr}(\mathbf{N})$ and $\mathrm{det}(\mathbf{N})$ as two quadratics in $q_0$.

    Substituting $q=q_0$ (so $c=1-hq_0 + h^2 q_0^2$) into $\mathrm{tr}(\mathbf{N})$ and $\mathrm{det}(\mathbf{N})$ gives:
    \begin{equation*}
        \mathrm{tr}(\mathbf{N}) = -h \beta_1 (h, q_0), 
        \mathrm{det}(\mathbf{N}) = -h^3q_0^2 \beta_2 (h, q_0),
    \end{equation*}
    where $\beta_1$ and $\beta_2$ are quadratics in $q_0$:
    \begin{align*}
        \beta_1(h, q_0) &= A_1 q_0^2 + B_1 q_0 + C_1, A_1 =h(2h^2 - 2h - 1), B_1 = 2(2h^2 -2h +1), C_1 = 4(h-1),\\
        \beta_2(h, q_0) &= A_2 q_0^2 + B_2 q_0 + C_2, A_2 = 2h^2 (h-1), B_2 = 2h^2 -2h +1, C_2 = 2(h-1).
    \end{align*}

    \noindent
    \textbf{Step 3}. We show that $\beta_2<0$ on $(h, q_0)\in (0,1/2]\times (0,1]$.

    For $h\in (0,1/2]$, $A_2=2h^2(h-1) < 0$, so $\beta_2$ is concave. Note that $\beta_2$ is increasing on $(0, q_0^*)$, where $q_0^*=-B_2/A_2$ is the vertex of its parabola; so if $q_0^* \geq 1$, the maximum over $(0,1]$ is at the right endpoint $q_0=1$.

    We can verify that for $\beta_2$, its vertex $q_0^*(\beta_2)\geq 1$. Indeed, $q_0^*(\beta_2)\geq 1 \iff 4h^3 -2h+1 \geq 0 \iff (2h-1)(2h^2-1)\geq 0$, which is true on $h\in (0,1/2)$. Hence, $\beta_2$ is increasing on all of $(0,1]$ for every $h$, with maximum at $q_0=1$: $\beta_2(h, 1)=2h^3 - 1< 0$. Thus we concluded that $\beta_2<0$ on $(h, q_0)\in (0,1/2]\times (0,1]$.\\

    \noindent
    \textbf{Step 4}. We show that $\beta_1<0$ on $(h, q_0)\in (0,1/2]\times (0,1]$.

    For $h\in (0,1/2]$, $A_1=h(2h^2 -2h-1)=h(2h(h-1)-1) < 0$, so $\beta_1$ is also concave. Unlike the case of $\beta_2$, the condition $q_0^*(\beta_1)\geq 1$ is not always true. Indeed, $q_0^*(\beta_1)\geq 1$ is equivalent to:
    \begin{equation*}
        -B_1 - 2 A_1 = -4 h^3 +6h -2 = -2 (h-1)(2h^2 +2h - 1) \leq 0 \iff 2h^2 +2h - 1\leq 0 \text{ for }h\in (0, 1/2).
    \end{equation*}
    So this requires that $h\leq h^* = \frac{\sqrt{3}-1}{2}\approx0.366$. Thus, we have to consider two cases:\\

    \textbf{Case 1} ($h\in (0, h^*)$): vertex $q_0^*\geq 1$ outside the domain; $\beta_1$ is increasing on $(0,1]$, maximum at $q_0=1$: $\beta_1(h, 1)=2h^3 +2h^2 -h - 2$. Its derivative $6h^2 +4h -1$ is negative at $h=0$ and positive at $h=1/2$, so $\beta(h, 1)$ decreases then increases on $[0,1/2]$, and its maximum on this interval occurs at an endpoint:
    \begin{equation*}
        \beta_1(0, 1)=-2 < 0, \beta_1(1/2, 1) = -7/4<0,
    \end{equation*}
    which means that $\beta_1(h, 1) < 0$.\\

     \textbf{Case 2} ($h\in [h^*,1/2]$): vertex $q_0^*< 1$ is interior; the maximum over $q\in(0,1]$ equals the unrestricted vertex:
     \begin{equation*}
         C_1 - \frac{B_1^2}{4A_1}= \frac{4h^4 - 8h^3 - 4h^2 +8h -1}{h(2h^2-2h-1)}.
     \end{equation*}
     The denominator $h(2h^2-2h-1)< 0$ throughout $(0, 1/2)$ (shown before). The numerator has real roots of around 0.1367, 0.8633, 2.0388 and is strictly positive between the first two. Since $h^*\approx 0.366$ lies in $(0.1367, 0.8633)$, we have numerator $> 0$ throughout $[h^*,1/2]$. Hence the vertex value is negative.

     From \textbf{Case 1} and \textbf{Case 2}, we have shown that $\beta_1(h, q_0) < 0$ on $(h, q_0)\in (0,1/2]\times (0,1]$.
\end{proof}

Now we are ready to give the proof of Theorem \ref{theorem2}.
\label{prooftheorem2}
\begin{proof}
We first rewrite the discrete process \eqref{eqn:discretupdate} in a more convenient way. By re-indexing the velocity component as $\mathbf{V}^{(k+1)}\mapsto\mathbf{V}^{(k)}$, \eqref{eqn:discretupdate} is equivalent to

\begin{align}
    \begin{split}
        \mathbf{Z}^{(k+1)} &=\mathbf{Z}^{(k)}+h \mathbf{V}^{(k)},\\
        \mathbf{X}^{(k+1)} &=M(\mathbf{Z}^{(k+1)})+ \mathbf{Q}^{-\frac{1}{2}}\boldsymbol{\epsilon}^{(k)},\\
        \mathbf{V}^{(k+1)} &=(1-h\gamma)\mathbf{V}^{(k)}\\
        &-hu\left[\beta \nabla g(\underbrace{\mathbf{Z}^{(k)}+h \mathbf{V}^{(k)}}_{\mathbf{Z}^{(k+1)}})+\frac{1}{\rho^2}\underbrace{\left(\mathbf{Z}^{(k)}+h \mathbf{V}^{(k)}-M(\mathbf{Z}^{(k)}+h \mathbf{V}^{(k)}) -\mathbf{Q}^{-1}\mathbf{W}^{(k)}\right)}_{\mathbf{Z}^{(k+1)}-\mathbf{X}^{(k+1)}}+\sqrt{2h\gamma u}\mathbf{N}^{(k)}\right].
    \end{split}
\end{align}
Again, we assume $(\mathbf{X}^{(k+1)}, \mathbf{Z}^{(k+1)}, \mathbf{V}^{(k+1)})$ and $(\widetilde{\mathbf{X}}^{(k+1)}, \widetilde{\mathbf{Z}}^{(k+1)}, \widetilde{\mathbf{V}}^{(k+1)})$ are synchronous in the sense that they are updated by the exact same Gaussian noises. By the definition of Wasserstein-2 distance,
\begin{align}
\label{eqn:W2upperbound}
\begin{split}
       \mathcal{W}^2_2\left(\bar{\mathbf{\Phi}}_h\mu^{(k)}, \bar{\mathbf{\Phi}}_h\widetilde{\mu}^{(k)} \right)
       &\leq \left(1 + \rho^{-4}\lVert\mathbf{Q}^{-1} \rVert^2_{\mathrm{op}} \right)\mathbb{E}\left[\left\lVert \mathbf{Z}^{(k)}-\widetilde{\mathbf{Z}}^{(k)} \right\rVert^2_2+\left\lVert \mathbf{V}^{(k)}-\widetilde{\mathbf{V}}^{(k)} \right\rVert^2_2\right].
    \end{split}
   \end{align}
Write $\Delta \mathbf{Z}^{(k)}=\mathbf{Z}^{(k)}-\widetilde{\mathbf{Z}}^{(k)}$, $\Delta \mathbf{V}^{(k)}=\mathbf{V}^{(k)}-\widetilde{\mathbf{V}}^{(k)}$, we aim to bound the coupling difference $\mathbb{E}\left[\left\lVert \Delta \mathbf{Z}^{(k)} \right\rVert^2_2+\left\lVert \Delta \mathbf{V}^{(k)} \right\rVert^2_2\right]$. Since both discrete processes share the same Gaussian noises, the noises in the difference cancel and it follows that
\begin{align}
\label{eqn:vectorform}
    \begin{split}
        \Delta \mathbf{V}^{(k+1)} &=\left(1-h\gamma -h^2 u \mathbf{H}^{(k)} \right)\Delta \mathbf{V}^{(k)} - h u \mathbf{H}^{(k)} \Delta \mathbf{Z}^{(k)},\\
        \Delta \mathbf{Z}^{(k+1)} &= \Delta \mathbf{Z}^{(k)} + h \Delta \mathbf{V}^{(k)},
    \end{split}
\end{align}
where 
 \begin{equation*}
       \mathbf{H}^{(k)} = \beta\left(\int_{0}^{1}\nabla^2 g(\mathbf{Z}^{(k)} + w(\widetilde{\mathbf{Z}}^{(k)}-\mathbf{Z}^{(k)}))dw\right)+\rho^{-2}\mathbf{I}_d-\rho^{-4}\mathbf{Q},
   \end{equation*} and $ \beta m_g\mathbf{I}_d \preceq\mathbf{H}^{(k)}\preceq (\beta M_g +\rho^{-2})\mathbf{I}_d$ (same as $\mathbf{H}_t$ in the proof of Theorem \ref{theorem1}).

   Define \[e_{k}=
   \begin{bmatrix}
       \Delta \mathbf{Z}^{(k)}+ \Delta  \mathbf{V}^{(k)}\\
        \Delta \mathbf{Z}^{(k)}
   \end{bmatrix},
   \] the matrix form of \eqref{eqn:vectorform} can be written as
   \[
       e_{k+1}= \mathbf{P}_k e_k, \text{ and } \lVert e_{k+1} \rVert^2_2 = e_k^\top\mathbf{P}_k^\top\mathbf{P}_ke_k \leq \lambda_{\mathrm{max}}\left(\mathbf{P}_k^\top\mathbf{P}_k \right) \lVert e_{k} \rVert^2_2,
   \]
   where 
   \[\mathbf{P}_k=
   \begin{bmatrix}
       (1-h(\gamma-1))\mathbf{I}_d-h^2 u \mathbf{H}^{(k)} & h(\gamma-1)\mathbf{I}_d+h(h-1)u \mathbf{H}^{(k)}\\
       h\mathbf{I}_d & (1-h)\mathbf{I}_d
   \end{bmatrix}.
   \]
   It suffices to show that $\lambda_{\mathrm{max}}\left(\mathbf{P}_k^\top\mathbf{P}_k \right) \leq 1-\frac{h}{\kappa}+\frac{h^2}{\kappa^2}$. Let $u=(\beta M_g +\rho^{-2})^{-1}$, $\gamma=2$, and diagonalize $\mathbf{H}^{(k)}$ with eigenvalues $\Lambda_j\in [\beta m_g, \beta M_g +1/\rho^2]$, and write 
   \begin{equation}
       \mathbf{P}_k^j =\begin{bmatrix}
           1-h-h^2 q & h-hq +h^2 q\\
           h & 1-h
       \end{bmatrix},
   \end{equation}
   where $q=\Lambda_j/(\beta M_g +1/\rho^2)\in [1/\kappa, 1]$ for $j=1,2,..,d$. It can be observe that
   \begin{equation*}
       \lambda_{\mathrm{max}}\left(\mathbf{P}_k^\top\mathbf{P}_k \right)=\max_{j=1,2,...,d} \lambda_{\mathrm{max}}\left((\mathbf{P}^j_k)^\top\mathbf{P}^j_k \right).
   \end{equation*}
   By Lemma \ref{lemma2}, we see that $\lambda_{\mathrm{max}}\left((\mathbf{P}^j_k)^\top\mathbf{P}^j_k \right) \leq 1-\frac{h}{\kappa}+\frac{h^2}{\kappa^2}$, which means that $ \lambda_{\mathrm{max}}\left(\mathbf{P}_k^\top\mathbf{P}_k \right) \leq 1-\frac{h}{\kappa}+\frac{h^2}{\kappa^2}$. Thus, 
   
   \begin{equation*}
       \lVert e_k\rVert^2_2\leq \left(1-\frac{h}{\kappa} +\frac{h^2}{\kappa^2} \right)^k \lVert e_0\rVert^2_2.
   \end{equation*}

   Combining with \eqref{eqn:W2upperbound} and the Rayleigh-quotient bound gives
   \begin{align*}
    \begin{split}
       \mathcal{W}^2_2\left(\mu^{(k)}, \widetilde{\mu}^{(k)} \right)
       &\leq \left(1 + \frac{1}{\rho^4}\lVert\mathbf{Q}^{-1} \rVert^2_{\mathrm{op}} \right)\mathbb{E}\left[\lVert\Delta \mathbf{Z}^{(k)} \rVert^2_2+\lVert\Delta \mathbf{V}^{(k)} \rVert^2_2 \right]\\
       &\leq 7\left(1 + \frac{1}{\rho^4}\lVert\mathbf{Q}^{-1} \rVert^2_{\mathrm{op}} \right)\left(1-\frac{h}{\kappa} + \frac{h^2}{\kappa^2} \right)^k\mathbb{E}\left[\lVert\Delta \mathbf{X}^{(0)} \rVert^2_2+\lVert\Delta \mathbf{Z}^{(0)} \rVert^2_2+\lVert\Delta \mathbf{V}^{(0)} \rVert^2_2 \right].
    \end{split}
   \end{align*}
   Finally, choosing the coupling of $(\mathbf{X}^{(0)}, \mathbf{Z}^{(0)}, \mathbf{V}^{(0)})\sim \mu^{(0)}$ and $(\widetilde{\mathbf{X}}^{(k)}, \widetilde{\mathbf{Z}}^{(0)}, \widetilde{\mathbf{V}}^{(0)})\sim\widetilde{\mu}^{(0)}$ to be the optimal for $\mathcal{W}_2$, so that $\mathcal{W}^2_2(\mu^{(0)}, \widetilde{\mu}^{(0)})=\mathbb{E}\left[\lVert \Delta\mathbf{X}^{(0)} \rVert^2_2 + \lVert \Delta\mathbf{Z}^{(0)} \rVert^2_2 + \lVert \Delta\mathbf{V}^{(0)} \rVert^2_2 \right]$, which concludes the proof.

\section{Proof of Bias Lemma \ref{lemma:bias}}
\label{proofofbiaslemma}
In this section, we aim to bound $\mathcal{W}^2_2(\Pi_\rho, \Pi_{\rho h})$. We first define the following frozen-parameter kinetic SDE.

\begin{defn}[Frozen-$\mathbf{x}$ kinetic SDE]
    For fixed $\mathbf{x}\in \mathbb{R}^d$, define on $(\mathbf{z}, \mathbf{v})\in \mathbb{R}^{2d}$, 
    \begin{align}
    \label{eqn:frozensde}
        \begin{split}
            d\widetilde{\mathbf{Z}}_t &= \widetilde{\mathbf{V}}_t dt,\\
            d\widetilde{\mathbf{V}}_t &= -\gamma \widetilde{\mathbf{V}}_t dt - u \left[\beta \nabla g(\widetilde{\mathbf{Z}}_t) + \rho^{-2}\left(\widetilde{\mathbf{Z}}_t-\mathbf{x} \right) \right] + \sqrt{2\gamma u} d\widetilde{\mathbf{B}}_t,
        \end{split}
    \end{align}
    where $(\widetilde{\mathbf{B}}_t)_{t\geq 0}$ is a Brownian motion.
\end{defn}
This is a standard, well-posed kinetic SDE with potential $U_{\mathbf{x}}(\mathbf{z})=\beta g(\mathbf{z}) + 1/2\rho^{-2}\lVert \mathbf{z}-\mathbf{x}\rVert^2_2$, where $\mathbf{x}$ is a fixed parameter. It is well-known that the stationary distribution of \eqref{eqn:frozensde} is $p(\mathbf{z}|\mathbf{x})\otimes \mathcal{N}(\mathbf{0}, u\mathbf{I}_d)$, where $p(\mathbf{z}|\mathbf{x})\propto e^{-U_{\mathbf{x}}(\mathbf{z})}$.

We next define the notion of idealized kernel and show that its invariance is $\Pi_\rho$.

\begin{defn}[Idealized kernel]
For $(\mathbf{x}, \mathbf{z}, \mathbf{v})\in\mathbb{R}^{3d}$, we perform one step of the idealized kernel $\boldsymbol{\Psi}_{\rho h}$ as follows:
\begin{enumerate}
    \item Run \eqref{eqn:frozensde} with fixed $\mathbf{x}$, starting from $(\mathbf{z}, \mathbf{v})$, for time $h$, yielding $(\widetilde{\mathbf{z}}, \widetilde{\mathbf{v}})$.
    \item Draw $\mathbf{x}^\prime \sim \mathcal{N}(M(\widetilde{\mathbf{z}}), \mathbf{Q}^{-1})$.
\end{enumerate}
\end{defn}
\begin{lemma}[Exact invariance]
\label{lemma:invariance}
    $\boldsymbol{\Psi}_{\rho h}\Pi_\rho = \Pi_\rho$.
\end{lemma}
\begin{proof}
    Let $(\mathbf{x}, \mathbf{z}, \mathbf{v})\sim \Pi_\rho$. Conditioning on $\mathbf{X}=\mathbf{x}$, the joint distribution $\Pi_\rho$ gives $(\mathbf{Z}, \mathbf{V})|\mathbf{X}=\mathbf{x}\sim p(\mathbf{z}|\mathbf{}x)\otimes \mathcal{N}(\mathbf{0}, u\mathbf{I}_d)$, which is precisely the stationary distribution of the SDE \eqref{eqn:frozensde}. Consequently, after running the SDE \eqref{eqn:frozensde} for time $h$, the pair $(\widetilde{\mathbf{Z}}, \widetilde{\mathbf{V}})|\mathbf{X}=\mathbf{x}$ has the same law. Drawing $\mathbf{X}^\prime\sim \mathcal{N}(M(\widetilde{\mathbf{z}}), \mathbf{Q}^{-1})=p(\mathbf{x}^\prime|\widetilde{\mathbf{z}})$ then reconstructs the full joint $(\mathbf{X}^\prime, \widetilde{\mathbf{Z}}, \widetilde{\mathbf{V}})\sim \Pi_\rho$.
\end{proof}

We define discrete times $t_k = kh$, for $k=0,1,2...$. We introduce the continuous-time process $(\widetilde{\mathbf{X}}_t, \widetilde{\mathbf{Z}}_t, \widetilde{\mathbf{V}}_t)$, defined piecewise on each interval $t\in [t_k, t_{k+1})$ by
\begin{align}
    \label{eqn:frozensde1}
        \begin{split}
            d\widetilde{\mathbf{Z}}_t &= \widetilde{\mathbf{V}}_t dt,\\
            d\widetilde{\mathbf{V}}_t &= -\gamma \widetilde{\mathbf{V}}_t dt - u \left[\beta \nabla g(\widetilde{\mathbf{Z}}_t) + \rho^{-2}\left(\widetilde{\mathbf{Z}}_t-\widetilde{\mathbf{X}}_{t_k} \right) \right] + \sqrt{2\gamma u} d\mathbf{B}_t,\\
            \widetilde{\mathbf{X}}_t &= \widetilde{\mathbf{X}}_{t_k}= M(\widetilde{\mathbf{Z}}_{t_k}) + \mathbf{Q}^{-\frac{1}{2}}\boldsymbol{\epsilon}_{t_k}.
        \end{split}
    \end{align}
Note that \eqref{eqn:frozensde1} is exactly the frozen kinetic SDE \eqref{eqn:frozensde} with the parameter $\mathbf{x}$ held fixed at the value $\widetilde{\mathbf{X}}_{t_k}$. We also define the interpolated process $(\bar{\mathbf{X}}_t, \bar{\mathbf{Z}}_t, \bar{\mathbf{V}}_t)$ associated with the discrete update \eqref{eqn:discretupdate} by
\begin{align}
\label{eqn:interpolation}
\begin{split}
    d\bar{\mathbf{V}}_t &= - \gamma \bar{\mathbf{V}}_{t_k}dt - u\left[\beta \nabla g(\bar{\mathbf{Z}}_{t_k}) +\rho^{-2}\left(\bar{\mathbf{Z}}_{t_k}-\bar{\mathbf{X}}_{t_k} \right) \right]dt + \sqrt{2\gamma u} d\mathbf{B}_t, \\
    d\bar{\mathbf{Z}}_t &= \bar{\mathbf{V}}_tdt + (t-t_k)\left(- \gamma \bar{\mathbf{V}}_{t_k} - u\left[\beta \nabla g(\bar{\mathbf{Z}}_{t_k}) +\rho^{-2}\left(\bar{\mathbf{Z}}_{t_k}-\bar{\mathbf{X}}_{t_k} \right) \right] \right)dt + (t-t_k)\sqrt{2\gamma u}d\mathbf{B}_t,\\
    \bar{\mathbf{X}}_t &= \bar{\mathbf{X}}_{t_k}= M(\bar{\mathbf{Z}}_{t_k}) + \mathbf{Q}^{-\frac{1}{2}}\boldsymbol{\epsilon}_{t_k},
\end{split}
\end{align}
with initial values $\bar{\mathbf{X}}_{t_k} = \mathbf{X}^{(k)}, \bar{\mathbf{Z}}_{t_k} = \mathbf{Z}^{(k)}, \bar{\mathbf{V}}_{t_k} = \mathbf{V}^{(k)}$. The first line of \eqref{eqn:interpolation} interpolates the velocity update; the second line interpolates $\mathbf{Z}$ so that, by Ito's formula,
\begin{equation*}
    \bar{\mathbf{Z}}_t = \bar{\mathbf{Z}}_{t_k} + (t- t_k)\bar{\mathbf{V}}_t.
\end{equation*}

The noise $\boldsymbol{\epsilon}_{t_k}$ is frozen at the beginning of each interval $[t_k, t_{k+1})$. We further assume that both processes start from the same point sampled from the stationary distribution $\Pi_\rho$,
\begin{equation*}
(\widetilde{\mathbf{X}}_{t_0},\widetilde{\mathbf{Z}}_{t_0},\widetilde{\mathbf{V}}_{t_0})
=(\bar{\mathbf{X}}_{t_0},\bar{\mathbf{Z}}_{t_0},\bar{\mathbf{V}}_{t_0})
=(\mathbf{X}^{(0)},\mathbf{Z}^{(0)},\mathbf{V}^{(0)})\sim\boldsymbol{\Pi}_\rho.
\end{equation*}
We couple the two processes \eqref{eqn:frozensde1} and \eqref{eqn:interpolation} synchronously, i.e., they are driven by the same Brownian motion $(\mathbf{B}_t)$ and the same Gaussian incrementals $\boldsymbol{\epsilon}_{t_k}$. The two processes \eqref{eqn:frozensde1} and \eqref{eqn:interpolation} admits $\Pi_\rho$ and $\Pi_{\rho h}$ as stationary distributions, respectively. Consequently, we may bound $\mathcal{W}^2_2(\Pi_\rho, \Pi_{\rho h})$ by comparing the two processes. 
Define the differences $\Delta \widetilde{\mathbf{X}}_t= \widetilde{\mathbf{X}}_t - \bar{\mathbf{X}}_t$, $\Delta \widetilde{\mathbf{Z}}_t= \widetilde{\mathbf{Z}}_t - \bar{\mathbf{Z}}_t$, $\Delta \widetilde{\mathbf{V}}_t= \widetilde{\mathbf{V}}_t - \bar{\mathbf{V}}_t$. A direct computation yields

\begin{align}
    \begin{split}
        d(\Delta \widetilde{\mathbf{V}}_t) &= -\gamma \widetilde{\mathbf{V}}_tdt - u \left[\beta \nabla g(\widetilde{\mathbf{Z}}_t) + \rho^{-2}\left(\widetilde{\mathbf{Z}}_t - \widetilde{\mathbf{X}}_{t_k}\right)\right] dt+ \gamma \bar{\mathbf{V}}_{t_k}dt + u \left[\beta \nabla g(\bar{\mathbf{Z}}_{t_k}) + \rho^{-2}\left(\bar{\mathbf{Z}}_{t_k} - \bar{\mathbf{X}}_{t_k} \right)\right]dt\\
        &= -\gamma \left(\widetilde{\mathbf{V}}_t  - \bar{\mathbf{V}}_t \right)dt - u \left[\left(\beta \nabla g(\widetilde{\mathbf{Z}}_t) +\rho^{-2}\widetilde{\mathbf{Z}}_t - \rho^{-4}\mathbf{Q}^{-1}\widetilde{\mathbf{Z}}_t \right)-\left(\beta \nabla g(\bar{\mathbf{Z}}_t) +\rho^{-2}\bar{\mathbf{Z}}_t - \rho^{-4}\mathbf{Q}^{-1}\bar{\mathbf{Z}}_t \right) \right]dt\\
        &\underbrace{- \gamma \left(\bar{\mathbf{V}}_t - \bar{\mathbf{V}}_{t_k} \right)dt - u \left[ \beta \nabla g(\bar{\mathbf{Z}}_t)-\beta \nabla g(\bar{\mathbf{Z}}_{t_k})  + \rho^{-2} (\bar{\mathbf{Z}}_t - \bar{\mathbf{Z}}_{t_k}) - \rho^{-4}\mathbf{Q}^{-1}(\bar{\mathbf{Z}}_t - \bar{\mathbf{Z}}_{t_k})\right]dt}_{d\mathcal{E}_t^{\mathbf{v},1}}\\
        &\underbrace{- u\rho^{-4} \mathbf{Q}^{-1} (\widetilde{\mathbf{Z}}_t- \widetilde{\mathbf{Z}}_{t_k})dt}_{d\mathcal{E}_t^{\mathbf{v},2}}\\
        &= -\gamma \Delta \widetilde{\mathbf{V}}_t dt - u \widetilde{\mathbf{H}}_t \Delta\widetilde{\mathbf{Z}}_t dt +  d\mathcal{E}_t^{\mathbf{v},1} + d\mathcal{E}_t^{\mathbf{v},2},
    \end{split}
\end{align}
where we used $\widetilde{\mathbf{X}}_{t_k}-\bar{\mathbf{X}}_{t_k}=\rho^{-2}\mathbf{Q}^{-1}(\widetilde{\mathbf{Z}}_{t_k}-\bar{\mathbf{Z}}_{t_k})$ in the second equality; $\widetilde{\mathbf{H}}_t  = \beta\left(\int_{0}^{1}\nabla^2 g(\widetilde{\mathbf{Z}}_t + w(\bar{\mathbf{Z}}_t-\widetilde{\mathbf{Z}}_t))dw\right)+\rho^{-2}\mathbf{I}_d-{\rho}^{-4}\mathbf{Q}^{-1}$ satisfies $ \beta m_g \mathbf{I}_d\preceq\widetilde{\mathbf{H}}_t\preceq (\beta M_g +\rho^{-2})\mathbf{I}_d$ 
(the same spectral bounds as $\mathbf{H}_t$ appearing in the proof of Theorem \ref{theorem1}).

Likewise,
\begin{align}
    \begin{split}
        d(\Delta \widetilde{\mathbf{Z}}_t) &= \widetilde{\mathbf{V}}_t dt - \bar{\mathbf{V}}_t dt\underbrace{- (t-t_k) \left(-\gamma \bar{\mathbf{V}}_{t_k} - u \left[\beta \nabla g(\bar{\mathbf{Z}}_{t_k})+\rho^{-2}\left(\bar{\mathbf{Z}}_{t_k}-\bar{\mathbf{X}}_{t_k} \right) \right] \right)dt- (t-t_k)\sqrt{2\gamma u}d\mathbf{B}_t}_{d\mathcal{E}^{z}_t}\\
        &= (\Delta\widetilde{\mathbf{V}}_t)dt + d\mathcal{E}^{\mathbf{z}}_t.
    \end{split}
\end{align}

We next upper-bound the second-moment rates of the error processes $\mathcal{E}_t^{\mathbf{v},1}$, $\mathcal{E}_t^{\mathbf{v},2}$, and $\mathcal{E}_t^{\mathbf{z}}$.

\subsection{Bounding $\mathbb{E}\left[\lVert d\mathcal{E}_t^{\mathbf{v},1}/dt  \rVert^2_2 \right]$ and $\mathbb{E}\left[\lVert d\mathcal{E}_t^{\mathbf{z}}/dt  \rVert^2_2 \right]$}

\begin{lemma}
\label{lemma1}
    Define 
    \begin{equation*}
        F(\mathbf{X}, \mathbf{Z})=\frac{1}{2}(\mathbf{A}\mathbf{X}-\mathbf{Y})^\top\mathbf{\Omega}(\mathbf{A}\mathbf{X}-\mathbf{Y}) + \beta g(\mathbf{Z})+\frac{1}{2}\rho^{-2}\lVert\mathbf{X}-\mathbf{Z} \rVert^2_2.
    \end{equation*}
    Let 
    \begin{align*}
        M_F&=\max\left\{\lambda_{max}(\mathbf{A}^\top\mathbf{\Omega}\mathbf{A}), \beta M_g\right\}+2\rho^{-2}, \\m_F&= \frac{1}{2}\left[(\lambda_{min}(\mathbf{A}^\top\mathbf{\Omega}\mathbf{A}) + \beta m_g + 2\rho^{-2}) - \sqrt{(\lambda_{min}(\mathbf{A}^\top\mathbf{\Omega}\mathbf{A})-\beta m_g)^2 + 4\rho^{-4}} \right]>0.
    \end{align*} Let $\left(\mathbf{X}^*, \mathbf{Z}^* \right)$ be a minimizer of $F$.Then $F$ is $m_F$-strongly convex and $M_F$-smooth. Moreover, for the interpolated process $\left(\bar{\mathbf{X}}_t, \bar{\mathbf{Z}}_t\right)$ defined above,
    \begin{align}
        \mathbb{E}\left[\lVert\beta \nabla g(\bar{\mathbf{Z}}_t) +\rho^{-2} (\bar{\mathbf{Z}}_t - \bar{\mathbf{X}}_t) \rVert^2_2\right] \leq 2 M_F \mathbb{E}\left[\lVert \bar{\mathbf{X}}_t - \mathbf{X}^* \rVert^2_2+\lVert \bar{\mathbf{Z}}_t - \mathbf{Z}^* \rVert^2_2\right] + \mathrm{tr}(\mathbf{Q}).
    \end{align}
\end{lemma}
\begin{proof}
    The Hessian blocks of $F$ are
    \begin{align*}
        \begin{split}
            \nabla_\mathbf{X}^2 F(\mathbf{X}, \mathbf{Z}) &= \mathbf{A}^\top\mathbf{\Omega}\mathbf{A} +\rho^{-2}\mathbf{I}_d,\\
            \nabla_\mathbf{Z}^2 F(\mathbf{X}, \mathbf{Z}) &= \beta \nabla^2g(\mathbf{Z}) + \rho^{-2}\mathbf{I}_d,\\
            \nabla_\mathbf{XZ}^2 F(\mathbf{X}, \mathbf{Z})&=\nabla_\mathbf{ZX}^2 F(\mathbf{X}, \mathbf{Z}) =-\rho^{-2}\mathbf{I}_d.
        \end{split}
    \end{align*}
    For any $(\mathbf{X}, \mathbf{Z})$,
    \begin{align*}
        \begin{bmatrix}
            \mathbf{X} \\
            \mathbf{Z}
        \end{bmatrix}^\top \nabla^2F(\mathbf{X}, \mathbf{Z})\ \begin{bmatrix}
            \mathbf{X} \\
            \mathbf{Z}
        \end{bmatrix}= \mathbf{X}^\top \mathbf{A}^\top\mathbf{\Omega}\mathbf{A}\mathbf{X} + \beta \mathbf{Z}^\top \nabla^2 g(\mathbf{Z})\mathbf{Z} + \rho^{-2}\lVert \mathbf{X}-\mathbf{Z}\rVert^2_2.
    \end{align*}
    The upper and lower bounds 
    \begin{align*}
          &\begin{bmatrix}
            \mathbf{X} \\
            \mathbf{Z}
        \end{bmatrix}^\top \nabla^2F(\mathbf{X}, \mathbf{Z}) \begin{bmatrix}
            \mathbf{X} \\
            \mathbf{Z}
        \end{bmatrix} \leq   M_F \left(\lVert \mathbf{X}\rVert^2_2 + \lVert \mathbf{Z}\rVert^2_2 \right), \\
        &\begin{bmatrix}
            \mathbf{X} \\
            \mathbf{Z}
        \end{bmatrix}^\top \nabla^2F(\mathbf{X}, \mathbf{Z}) \begin{bmatrix}
            \mathbf{X} \\
            \mathbf{Z}
        \end{bmatrix} \geq   m_F \left(\lVert \mathbf{X}\rVert^2_2 + \lVert \mathbf{Z}\rVert^2_2 \right),
    \end{align*}
    where $m_F$ is the smallest eigenvalue of 
    $\begin{bmatrix}
            \lambda_{min}(\mathbf{A}^\top\mathbf{\Omega}\mathbf{A})+\rho^{-2} & -\rho^{-2} \\
            -\rho^{-2} & \beta m_g + \rho^{-2}
        \end{bmatrix},$ whose determinant is $\lambda_{min}(\mathbf{A}^\top\mathbf{\Omega}\mathbf{A})\beta m_g + (\lambda_{min}(\mathbf{A}^\top\mathbf{\Omega}\mathbf{A})+\beta m_g)\rho^{-2}>0$.
        
    Next, observe that
    \begin{align*}
        \nabla_{\mathbf{X}} F(\bar{\mathbf{X}}_t, \bar{\mathbf{Z}}_t) + \nabla_{\mathbf{Z}} F(\bar{\mathbf{X}}_t, \bar{\mathbf{Z}}_t) &=  \mathbf{A}^\top \mathbf{\Omega}(\mathbf{A}\bar{\mathbf{X}}_t- \mathbf{Y}) + \beta \nabla g(\bar{\mathbf{Z}}_t) + \rho^{-2} \left(\bar{\mathbf{X}}_t - \bar{\mathbf{Z}}_t \right) + \rho^{-2} \left(\bar{\mathbf{Z}}_t - \bar{\mathbf{X}}_t \right)\\
        &= \beta \nabla g(\bar{\mathbf{Z}}_t) + \mathbf{Q}\bar{\mathbf{X}}_t - \sigma^{-2}\mathbf{A}^\top \mathbf{Y} - \rho^{-2}\bar{\mathbf{X}}_t.
    \end{align*}
    Because $\bar{\mathbf{X}}_t = M(\bar{\mathbf{Z}}_t)+\mathbf{Q}^{-\frac{1}{2}}\boldsymbol{\epsilon}_{t_k}$, we have the identity
        $\mathbf{Q}\bar{\mathbf{X}}_t =  \mathbf{A}^\top\mathbf{\Omega} \mathbf{Y} + \rho^{-2}\bar{\mathbf{Z}}_t + \mathbf{Q}^{1/2}\boldsymbol{\epsilon}_{t_k}$, so 
    \begin{align*}
        \nabla_{\mathbf{X}} F(\bar{\mathbf{X}}_t, \bar{\mathbf{Z}}_t) + \nabla_{\mathbf{Z}} F(\bar{\mathbf{X}}_t, \bar{\mathbf{Z}}_t) 
        &= \beta \nabla g(\bar{\mathbf{Z}}_t) + \frac{1}{\rho^2} \left(\bar{\mathbf{Z}}_t - \bar{\mathbf{X}}_t \right) + \mathbf{Q}^{\frac{1}{2}}\boldsymbol{\epsilon}_{t_k}.
    \end{align*}

It follows that
\begin{align*}
    \begin{split}
        &\mathbb{E}\left[\lVert\beta \nabla g(\bar{\mathbf{Z}}_t) +\rho^{-2} (\bar{\mathbf{Z}}_t - \bar{\mathbf{X}}_t) \rVert^2_2\right] = \mathbb{E}\left[\lVert\nabla_{\mathbf{X}}F(\bar{\mathbf{X}}_t, \bar{\mathbf{Z}}_t) + \nabla_{\mathbf{Z}}F(\bar{\mathbf{X}}_t, \bar{\mathbf{Z}}_t) - \mathbf{Q}^{1/2}\boldsymbol{\epsilon}_{t_k}\rVert^2_2\right] \\
        &=\mathbb{E}\left[\lVert\nabla_{\mathbf{X}}F(\bar{\mathbf{X}}_t, \bar{\mathbf{Z}}_t) + \nabla_{\mathbf{Z}}F(\bar{\mathbf{X}}_t, \bar{\mathbf{Z}}_t) \rVert^2_2\right] + \mathrm{tr}(\mathbf{Q}) - 2\mathbb{E}\langle \nabla_{\mathbf{X}}F(\bar{\mathbf{X}}_t, \bar{\mathbf{Z}}_t) + \nabla_{\mathbf{Z}}F(\bar{\mathbf{X}}_t, \bar{\mathbf{Z}}_t),\mathbf{Q}^{1/2}\boldsymbol{\epsilon}_{t_k} \rangle.
    \end{split}
    \end{align*}
    Note that the cross-term is simplified as
    \begin{align*}
        \mathbb{E}\langle \nabla_{\mathbf{X}}F(\bar{\mathbf{X}}_t, \bar{\mathbf{Z}}_t) + \nabla_{\mathbf{Z}}F(\bar{\mathbf{X}}_t, \bar{\mathbf{Z}}_t),\mathbf{Q}^{1/2}\boldsymbol{\epsilon}_{t_k} \rangle= \mathbb{E}\langle -\rho^{-2}\mathbf{Q}^{-1/2}\boldsymbol{\epsilon}_{t_k} + \mathbf{Q}^{1/2}\boldsymbol{\epsilon}_{t_k}, \mathbf{Q}^{1/2}\boldsymbol{\epsilon}_{t_k} \rangle=-\rho^{-2}d + \mathrm{tr}(\mathbf{Q}).
    \end{align*}
    Consequently, 
    \begin{align*}
    \begin{split}
        \mathbb{E}\left[\lVert\beta \nabla g(\bar{\mathbf{Z}}_t) +\rho^{-2} (\bar{\mathbf{Z}}_t - \bar{\mathbf{X}}_t) \rVert^2_2\right] = \mathbb{E}\left[\lVert\nabla_{\mathbf{X}}F(\bar{\mathbf{X}}_t, \bar{\mathbf{Z}}_t) + \nabla_{\mathbf{Z}}F(\bar{\mathbf{X}}_t, \bar{\mathbf{Z}}_t) \rVert^2_2\right] + \mathrm{tr}(\mathbf{Q}) - 2(-\rho^{-2}d +\mathrm{tr}(\mathbf{Q})). 
    \end{split}
    \end{align*}
    Using $\mathrm{tr}(\mathbf{Q})=\mathrm{tr}(\mathbf{A}^\top\mathbf{\Omega}\mathbf{A})+\rho^{-2}d\geq \rho^{-2}d$ gives the claimed bound.
\end{proof}

\begin{lemma}
\label{lemma:uniformmomentbound}
    Let $\mu^{(0)}$ be an arbitrary distribution on $\mathbb{R}^{3d}$ with $\left(\mathbf{X}^{(0)}, \mathbf{Z}^{(0)}, \mathbf{V}^{(0)} \right)\sim\mu^{(0)}$, and write $\bar{\mathbf{\Phi}}_h^k\mu^{(0)}$ for the law of $\left(\mathbf{X}^{(k)}, \mathbf{Z}^{(k)}, \mathbf{V}^{(k)} \right)$ generated by the discrete recursion \eqref{eqn:discretupdate}. For $h\in(0,1/2]$ one has
    \begin{equation*}
        \Theta^* = \sup_{k\in\mathbb{N}, h\in(0,1/2]}\left\{\mathbb{E}\lVert\mathbf{X}^{(k)}-\mathbf{X}^* \rVert^2_2 + \mathbb{E}\lVert\mathbf{Z}^{(k)}-\mathbf{Z}^* \rVert^2_2 + \mathbb{E}\lVert\mathbf{V}^{(k)}\rVert^2_2 \right\}< \infty.
    \end{equation*}
\end{lemma}
\begin{proof}
    By the discrete contraction Corollary \ref{cor:discrete-convergence} and $1-h/\kappa + h^2/\kappa^2 \leq 1$,
    \begin{align}
    \label{eqn:bounded}
    \begin{split}
        \sup_{k\in\mathbb{N}, h\in(0,1/2]} \mathcal{W}^2_2(\bar{\mathbf{\Phi}}^k_h\mu^{(0)}, \Pi_{\rho}) &\leq \sup_{k\in\mathbb{N},h\in(0,1/2]}\left\{2\mathcal{W}^2_2(\bar{\mathbf{\Phi}}^k_h\mu^{(0)}, \Pi_{\rho h}) + 2\mathcal{W}^2_2(\Pi_{\rho}, \Pi_{\rho h})\right\}\\
        &\leq \sup_{h\in(0,1/2]}\left\{ 2 C\mathcal{W}^2_2(\mu^{(0)}, \Pi_{\rho h}) + 2\mathcal{W}^2_2(\Pi_{\rho}, \Pi_{\rho h})\right\}<\infty.
    \end{split}
    \end{align}
    Choosing an optimal coupling between $\left(\hat{\mathbf{X}}, \hat{\mathbf{Z}}, \hat{\mathbf{V}}\right)\sim \Pi_{\rho}$ and $\left(\mathbf{X}^{(k)}, \mathbf{Z}^{(k)}, \mathbf{V}^{(k)} \right)\sim \bar{\mathbf{\Phi}}^k_h\mu^{(0)}$, we obtain
    \begin{align*}
        &\mathbb{E}\lVert\mathbf{X}^{(k)}-\mathbf{X}^* \rVert^2_2 + \mathbb{E}\lVert\mathbf{Z}^{(k)}-\mathbf{Z}^* \rVert^2_2 + \mathbb{E}\lVert\mathbf{V}^{(k)}\rVert^2_2\leq 2\left( \mathbb{E}\lVert\mathbf{X}^{(k)}-\hat{\mathbf{X}} \rVert^2_2 + \mathbb{E}\lVert\mathbf{Z}^{(k)}-\hat{\mathbf{Z}} \rVert^2_2 + \mathbb{E}\lVert\mathbf{V}^{(k)}-\hat{\mathbf{V}}\rVert^2_2\right)\\
        & + 2\left(\mathbb{E}\lVert\hat{\mathbf{X}}-\mathbf{X}^* \rVert^2_2 + \mathbb{E}\lVert\hat{\mathbf{Z}}-\mathbf{Z}^* \rVert^2_2 + \mathbb{E}\lVert\hat{\mathbf{V}}\rVert^2_2\right)= 2\mathcal{W}^2_2(\bar{\mathbf{\Phi}}^k_h\mu^{(0)}, \Pi_{\rho}) + 2ud + \frac{4d}{m_F},
    \end{align*}
    where we used the standard moment bounds $\mathbb{E}\lVert\hat{\mathbf{X}}-\mathbf{X}^* \rVert^2_2 + \mathbb{E}\lVert\hat{\mathbf{Z}}-\mathbf{Z}^* \rVert^2_2\leq \frac{2d}{m_F}$ \citep[Proposition 1]{durmus2019high} and $\mathbb{E}\lVert\hat{\mathbf{V}}\rVert^2_2=ud$. Combining this with \eqref{eqn:bounded} finishes the proof.
\end{proof}

We now bound the error term $\mathbb{E}\lVert d\mathcal{E}^{\mathbf{v}1}_t/dt \rVert^2_2$. Write
\begin{align}
    \begin{split}
        d\mathcal{E}^{\mathbf{v},1}_t/dt &= -\gamma (\bar{\mathbf{V}}_{t}-\bar{\mathbf{V}}_{t_k})\\
        &- u \left[ \beta \nabla g(\bar{\mathbf{Z}}_{t})-\beta \nabla g(\bar{\mathbf{Z}}_{t_k}) + \rho^{-2}\left(\bar{\mathbf{Z}}_{t}- \bar{\mathbf{Z}}_{t_k} \right) - \rho^{-4}\mathbf{Q}^{-1}\left(\bar{\mathbf{Z}}_{t} - \bar{\mathbf{Z}}_{t_k} \right) \right]\\
        &= -\gamma (\bar{\mathbf{V}}_{t}-\bar{\mathbf{V}}_{t_k}) - u \mathbf{G}_t\left( \bar{\mathbf{Z}}_t-\bar{\mathbf{Z}}_{t_k}\right),
    \end{split}
\end{align}
where $\mathbf{G}_t = \beta\left(\int_{0}^{1}\nabla^2 g(\bar{\mathbf{Z}}_{t_k} + w(\bar{\mathbf{Z}}_t-\bar{\mathbf{Z}}_{t_k}))dw\right)+\rho^{-2}\mathbf{I}_d-{\rho}^{-4}\mathbf{Q}^{-1}$ satisfies the same spectral bounds $ \beta m_g \preceq\mathbf{G}_t\preceq \beta M_g +\rho^{-2}$ as $\mathbf{H}_t$. Therefore
\begin{align}
    \begin{split}
        \mathbb{E}\left[\bigg\lVert \frac{d\mathcal{E}^{\mathbf{v},1}_t}{dt} \bigg\rVert^2_2 \right]
        \leq 2\gamma^2 \mathbb{E}\lVert \bar{\mathbf{V}}_{t}-\bar{\mathbf{V}}_{t_k} \rVert^2_2 + 2\mathbb{E}\lVert \bar{\mathbf{Z}}_{t}-\bar{\mathbf{Z}}_{t_k} \rVert^2_2. 
    \end{split}
\end{align}
We bound $\mathbb{E}\lVert \bar{\mathbf{Z}}_{t}-\bar{\mathbf{Z}}_{t_k} \rVert^2$ as follows. From the interpolation formula,
\begin{align}
\label{eqn:sup1}
\begin{split}
    \mathbb{E}\lVert \bar{\mathbf{Z}}_{t}-\bar{\mathbf{Z}}_{t_k} \rVert^2 &=(t-t_k)^2 \mathbb{E}\lVert\bar{\mathbf{V}}_t \rVert^2_2\\
    &= (t-t_k)^2 \mathbb{E}\bigg\lVert \bar{\mathbf{V}}_{t_k} + (t-t_k)\left[-\gamma \bar{\mathbf{V}}_{t_k} - u \left(\beta \nabla g(\bar{\mathbf{Z}}_{t_k}) +\frac{1}{\rho^2}\left(\bar{\mathbf{Z}}_{t_k}- \bar{\mathbf{X}}_{t_k}\right)\right) +\sqrt{2\gamma u} (\mathbf{B}_t-\mathbf{B}_{t_k}) \right] \bigg\rVert^2_2\\
    &= (t-t_k)^2 \mathbb{E}\bigg\lVert \bar{\mathbf{V}}_{t_k} + (t-t_k)\left[-\gamma \bar{\mathbf{V}}_{t_k} - u \left(\beta \nabla g(\bar{\mathbf{Z}}_{t_k}) +\frac{1}{\rho^2}\left(\bar{\mathbf{Z}}_{t_k}- \bar{\mathbf{X}}_{t_k}\right)\right) \right] \bigg\rVert^2_2 + 2 (t-t_k)^3 \gamma u d\\
    &\leq 2(t-t_k)^2 \left[(2\gamma^2 + 1)\mathbb{E}\lVert\bar{\mathbf{V}}_{t_k} \rVert^2_2 + 2 u^2\mathbb{E}\bigg\lVert \beta \nabla g(\bar{\mathbf{Z}}_{t_k})+\frac{1}{\rho^2}\left(\bar{\mathbf{Z}}_{t_k}-\bar{\mathbf{X}}_{t_k}\right)\bigg\rVert^2_2 \right] + 2(t-t_k)^3 \gamma u d\\
    &\leq 2(t-t_k)^2 \left[ (2\gamma^2 + 1)\mathbb{E}\lVert\bar{\mathbf{V}}_{t_k} \rVert^2_2 + 4u^2 M_F \mathbb{E}\lVert \bar{\mathbf{X}}_{t_k} - \mathbf{X}^*  \rVert^2_2 + 4u^2 M_F \mathbb{E}\lVert \bar{\mathbf{Z}}_{t_k} - \mathbf{Z}^*  \rVert^2_2  \right]\\
    &+ 4u^2 \mathrm{tr}(\mathbf{Q}) (t-t_k)^2 + 2(t-t_k)^3 \gamma u d\\
    & \leq 2 \max\left(2\gamma^2 +1, 4u^2 M_F \right) (t-t_k)^2 \left(\mathbb{E}\lVert\bar{\mathbf{V}}_{t_k} \rVert^2_2 + \mathbb{E}\lVert \bar{\mathbf{X}}_{t_k} - \mathbf{X}^*  \rVert^2_2 +\mathbb{E}\lVert \bar{\mathbf{Z}}_{t_k} - \mathbf{Z}^*  \rVert^2_2  \right)\\
    &+ 4u^2 \mathrm{tr}(\mathbf{Q}) (t-t_k)^2 + 2\gamma u d(t-t_k)^3\\
    &\leq \underbrace{\left[2 \max\left(2\gamma^2 +1, 4u^2 M_F \right) \Theta^* + 4u^2 \mathrm{tr}(\mathbf{Q}) + 2\gamma u d \right]}_{C_1}(t-t_k),
\end{split}
\end{align}
where we used Lemma \ref{lemma1}, the uniform moment bound $\Theta^{*}$ of Lemma \ref{lemma:uniformmomentbound}, and the elementary estimate $t-t_k\leq h\leq 1$. 


An analogous calculation for the velocity increment yields
\begin{align}
\label{eqn:sup2}
\begin{split}
    \mathbb{E}\lVert\bar{\mathbf{V}}_t-\bar{\mathbf{V}}_{t_k} \rVert^2_2 &= (t-t_k)^2 \mathbb{E}\bigg\lVert -\gamma\bar{\mathbf{V}}_{t_k}  - u \left[ \beta \nabla g(\bar{\mathbf{Z}}_{t_k}) + \frac{1}{\rho^2}\left(\bar{\mathbf{Z}}_{t_k}  -\bar{\mathbf{X}}_{t_k} \right)\right]  \bigg\rVert^2_2 + 2\gamma u d(t-t_k)\\
    & \leq 2(t-t_k)^2 \left[ \gamma^2\mathbb{E}\lVert \bar{\mathbf{V}}_{t_k}  \rVert^2_2 + u^2 \mathbb{E}\bigg\lVert \beta \nabla g(\bar{\mathbf{Z}}_{t_k})+\frac{1}{\rho^2}\left(\bar{\mathbf{Z}}_{t_k}-\bar{\mathbf{X}}_{t_k}\right)\bigg\rVert^2_2 \right]+ 2\gamma u d (t-t_k)\\
    &\leq 2(t-t_k)^2 \left[ \gamma^2\mathbb{E}\lVert \bar{\mathbf{V}}_{t_k}  \rVert^2_2 + u^2 M_F \mathbb{E}\lVert\bar{\mathbf{X}}_{t_k}- \mathbf{X}^* \rVert^2_2 + u^2 M_F \mathbb{E}\lVert\bar{\mathbf{Z}}_{t_k}- \mathbf{Z}^* \rVert^2_2\right] + 2\gamma u d(t-t_k)\\
    & \leq 2 \max \left( \gamma^2, u^2 M_F\right)(t-t_k)^2 \left( \mathbb{E}\lVert\bar{\mathbf{V}}_{t_k} \rVert^2_2 + \mathbb{E}\lVert \bar{\mathbf{X}}_{t_k} - \mathbf{X}^*  \rVert^2_2 +  \mathbb{E}\lVert \bar{\mathbf{Z}}_{t_k} - \mathbf{Z}^*  \rVert^2_2  \right) + 2\gamma u d (t-t_k)\\
    &\leq \underbrace{\left[2 \max \left( \gamma^2, u^2 M_F\right) \Theta^{*} + 2\gamma u \right]}_{C_2}(t-t_k)
\end{split}
\end{align}
Combining \eqref{eqn:sup1} and \eqref{eqn:sup2} we obtain
\begin{align}
\label{eqn:supp2}
     \mathbb{E}\left[\bigg\lVert \frac{d\mathcal{E}^{\mathbf{v},1}_t}{dt} \bigg\rVert^2_2 \right] \leq C_3 (t-t_k),
\end{align}
where $C_3 = 2C_1 + 2\gamma^2 C_2$. For the positional error we simply have
\begin{align}
\label{eqn:supp3}
    \mathbb{E}\left[\bigg\lVert \frac{d\mathcal{E}^{\mathbf{z}}_t}{dt} \bigg\rVert^2_2 \right]= \mathbb{E}\lVert \bar{\mathbf{V}}_t - \bar{\mathbf{V}}_{t_k}\rVert^2_2\leq C_2 (t-t_k).
\end{align}

\subsection{Bounding $\mathbb{E}\left[\lVert d\mathcal{E}_t^{\mathbf{v},2}/dt  \rVert^2_2 \right]$}
We bound the second error term as follows
\begin{align}
\label{eqn:supp44}
    \begin{split}
        \mathbb{E}\left[\bigg\lVert \frac{d\mathcal{E}_t^{\mathbf{v},2}}{dt}  \bigg\rVert^2_2 \right] &\leq u^2 \rho^{-8} \lVert \mathbf{Q}^{-1}\rVert^2_{\mathrm{op}} \mathbb{E} \lVert \widetilde{\mathbf{Z}}_t - \widetilde{\mathbf{Z}}_{t_k}\rVert^2_2\leq u^2 \rho^{-8} \lVert \mathbf{Q}^{-1}\rVert^2_{\mathrm{op}} \mathbb{E}\bigg\lVert \int_{t_k}^t \widetilde{\mathbf{V}}_s ds \bigg\rVert^2_2\\
        &\leq u^2 \rho^{-8} \lVert \mathbf{Q}^{-1}\rVert^2_{\mathrm{op}} (t-t_k)\int_{t_k}^t \mathbb{E}\lVert \widetilde{\mathbf{V}}_s\rVert^2_2 ,ds
    \end{split}
\end{align}
where we used Cauchy-Schwarz inequality in the second line. Since we assume that the idealized chain starts from the stationary distribution $\Pi_\rho$, by Lemma \ref{lemma:invariance} (exact invariance), the idealized chain has law $\Pi_\rho$ at every discrete times. So $\widetilde{\mathbf{V}_s}\sim\mathcal{N}(\mathbf{0}, u\mathbf{I}_d)$ and $\mathbb{E}\lVert\widetilde{\mathbf{V}_s}  \rVert^2_2 = ud$ for $s\in[t_k, t_{k+1})$.
Consequently,
\begin{align}
\label{eqn:supp4}
    \begin{split}
        \mathbb{E}\left[\bigg\lVert \frac{d\mathcal{E}_t^{\mathbf{v},2}}{dt}  \bigg\rVert^2_2 \right]\leq u^2 d \rho^{-8} \lVert \mathbf{Q}^{-1}\rVert^2_{\mathrm{op}} (t-t_k)^2\leq C_4 (t-t_k),
    \end{split}
\end{align}
where $C_4 = u^2 d \rho^{-8} \lVert \mathbf{Q}^{-1}\rVert^2_{\mathrm{op}}$.
\subsection{Closing the bias estimate}

Following the same Lyapunov analysis as in the proof of the main contraction theorem, we obtain the differential inequality
\begin{align}
\label{eqn:supp5}
\begin{split}
    \frac{d}{dt} \left(\lVert\Delta\widetilde{\mathbf{Z}}_t +\Delta\widetilde{\mathbf{V}}_t  \rVert^2 + \lVert \Delta \widetilde{\mathbf{Z}}_t\rVert^2 \right) &= -2 \begin{bmatrix}
        \Delta\widetilde{\mathbf{Z}}_t^\top+\Delta\widetilde{\mathbf{V}}_t^\top & \Delta\widetilde{\mathbf{Z}}_t^\top
    \end{bmatrix}\mathbf{R}_t
    \begin{bmatrix}
        \Delta\widetilde{\mathbf{Z}}_t+\Delta\widetilde{\mathbf{V}}_t \\ \Delta\widetilde{\mathbf{Z}}_t 
    \end{bmatrix}\\
    &- 2\langle \Delta\widetilde{\mathbf{Z}}_t +\Delta\widetilde{\mathbf{V}}_t, \frac{d \mathcal{E}^{\mathbf{v}1}_{t}}{dt}+\frac{d \mathcal{E}^{\mathbf{v}2}_{t}}{dt} +\frac{d \mathcal{E}^{\mathbf{z}}_{t}}{dt}\rangle - 2\langle \Delta\widetilde{\mathbf{Z}}_t , \frac{d \mathcal{E}^{\mathbf{z}}_{t}}{dt} \rangle\\
    & \leq - \frac{1}{\kappa} \left(\lVert\Delta\widetilde{\mathbf{Z}}_t +\Delta\widetilde{\mathbf{V}}_t  \rVert^2 + \lVert \Delta \widetilde{\mathbf{Z}}_t\rVert^2_2 \right) \\
    &+ \frac{1}{2\kappa}\lVert\Delta\widetilde{\mathbf{Z}}_t +\Delta\widetilde{\mathbf{V}}_t  \rVert^2_2 + 2\kappa \bigg\lVert \frac{d \mathcal{E}^{\mathbf{v}1}_{t}}{dt}+\frac{d \mathcal{E}^{\mathbf{v}2}_{t}}{dt} + \frac{d \mathcal{E}^{\mathbf{z}}_{t}}{dt}\bigg\rVert^2_2 \\
    &+ \frac{1}{2\kappa}\lVert\Delta\widetilde{\mathbf{Z}}_t  \rVert^2_2 + 2\kappa \bigg\lVert \frac{d \mathcal{E}^{\mathbf{z}}_{t}}{dt}\bigg\rVert^2_2 \\
    &\leq - \frac{1}{2\kappa} \left(\lVert\Delta\widetilde{\mathbf{Z}}_t +\Delta\widetilde{\mathbf{V}}_t  \rVert^2_2 + \lVert \Delta \widetilde{\mathbf{Z}}_t\rVert^2 \right) \\
    &+ 6\kappa \bigg\lVert \frac{d \mathcal{E}^{\mathbf{v},1}_{t}}{dt}\bigg\rVert^2_2+ 6\kappa \bigg\lVert \frac{d \mathcal{E}^{\mathbf{v},2}_{t}}{dt}\bigg\rVert^2_2 + 8\kappa \bigg\lVert \frac{d \mathcal{E}^\mathbf{z}_{t}}{dt}\bigg\rVert^2_2
\end{split}
\end{align}

\begin{figure*}[t]
    \centering
    \includegraphics[width=\textwidth]{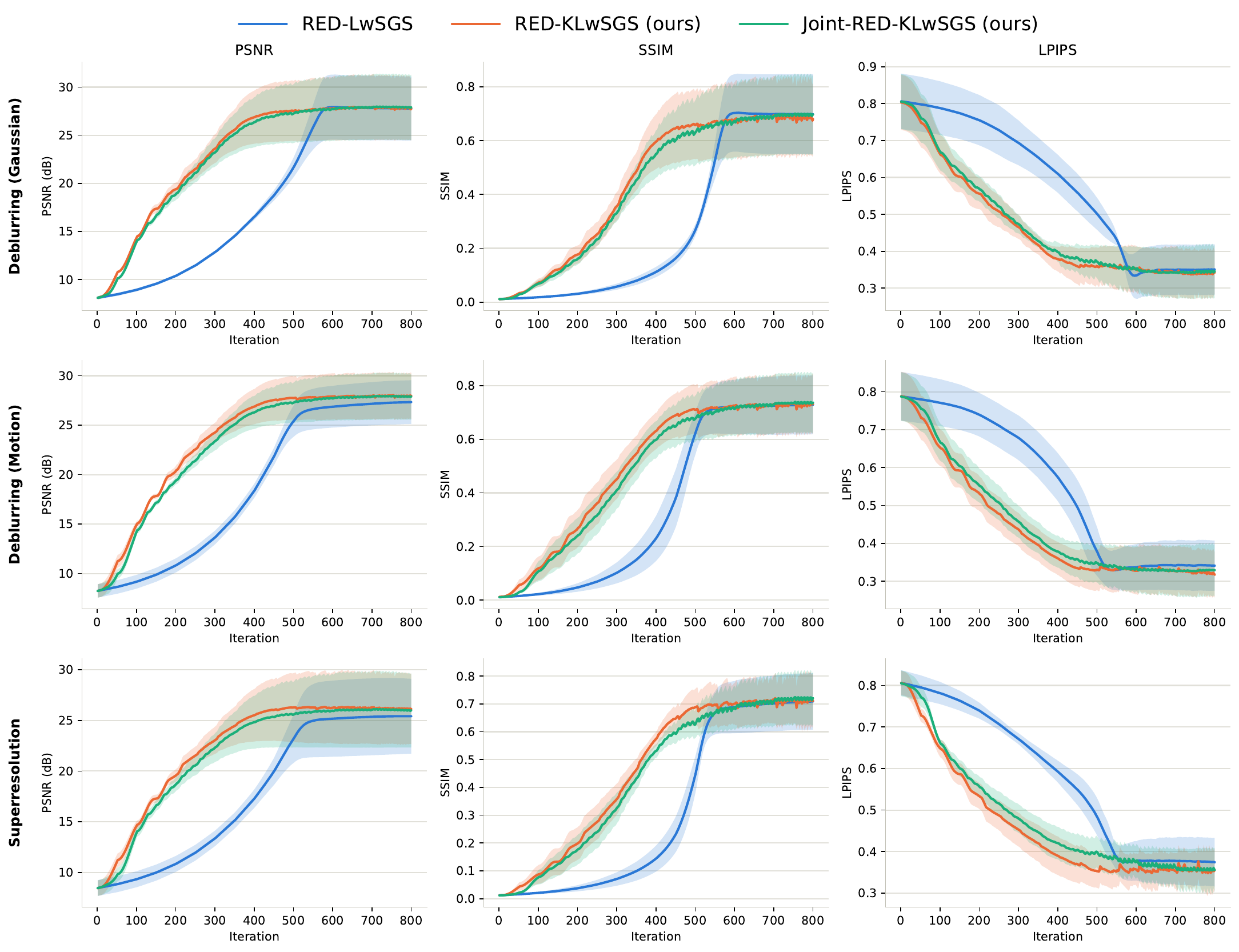}
    \caption{Convergence of RED-LwSGS (baseline), RED-KLwSGS (ours), and
    Joint-RED-KLwSGS (ours) for Gaussian deblurring, motion deblurring,
    and super-resolution (rows), on ten ImageNet images, measured by PSNR,
    SSIM, and LPIPS (columns).}
    \label{fig:convergence_imagenet}
\end{figure*}

Taking expectations and inserting the bounds \eqref{eqn:supp2}, \eqref{eqn:supp3}, \eqref{eqn:supp4} and \eqref{eqn:supp5} yields
\begin{align*}
    \frac{d}{dt} \mathbb{E}\left[\lVert\Delta\widetilde{\mathbf{Z}}_t +\Delta\widetilde{\mathbf{V}}_t  \rVert^2 + \lVert \Delta \widetilde{\mathbf{Z}}_t\rVert^2 \right] &\leq -\frac{1}{2\kappa}\mathbb{E}\left[\lVert\Delta\widetilde{\mathbf{Z}}_t +\Delta\widetilde{\mathbf{V}}_t  \rVert^2 + \lVert \Delta \widetilde{\mathbf{Z}}_t\rVert^2 \right] + C_5 (t-t_k)
\end{align*}
where $C_5= 6\kappa C_3 + 6\kappa C_4 + 8\kappa C_2$. Applying Grönwall's inequality on each interval $[t_k, t_{k+1}]$ and writing $r=e^{-\frac{h}{2\kappa}}$, we obtain
\begin{align*}
\mathbb{E}\left[\lVert\Delta \widetilde{\mathbf{Z}}_{t_{k}} + \Delta\widetilde{\mathbf{V}}_{t_{k}} \rVert^2_2 + \lVert\Delta\widetilde{\mathbf{Z}}_{t_{k}} \rVert^2_2 \right]
&\le r\,\mathbb{E}\left[\lVert\Delta \widetilde{\mathbf{Z}}_{t_{k-1}} + \Delta\widetilde{\mathbf{V}}_{t_{k-1}} \rVert^2_2 + \lVert\Delta\widetilde{\mathbf{Z}}_{t_{k-1}} \rVert^2_2 \right]
 + C_5\Big[\,2\kappa h-4\kappa^2\big(1-e^{-h/(2\kappa)}\big)\Big] \notag\\
&\le r\,\mathbb{E}\left[\lVert\Delta \widetilde{\mathbf{Z}}_{t_{k-1}} + \Delta\widetilde{\mathbf{V}}_{t_{k-1}} \rVert^2_2 + \lVert\Delta\widetilde{\mathbf{Z}}_{t_{k-1}} \rVert^2_2 \right] + \frac{C_5}{2}\,h^2
\qquad\text{(using } 1-e^{-x}\ge x-\tfrac{x^2}{2}\text{)} \notag\\
&\le r^k\,\mathbb{E}\left[\lVert\Delta \widetilde{\mathbf{Z}}_{t_{0}} + \Delta\widetilde{\mathbf{V}}_{t_{0}} \rVert^2_2 + \lVert\Delta\widetilde{\mathbf{Z}}_{t_{0}} \rVert^2_2 \right]
 + \frac{C_5h^2}{2}\cdot\frac{1-r^k}{1-r} \notag\\
&\le r^k\,\mathbb{E}\left[\lVert\Delta \widetilde{\mathbf{Z}}_{t_{0}} + \Delta\widetilde{\mathbf{V}}_{t_{0}} \rVert^2_2 + \lVert\Delta\widetilde{\mathbf{Z}}_{t_{0}} \rVert^2_2 \right] + 2\kappa C_5 h.
\end{align*}

Starting both processes from the same initial condition distributed according to $\Pi_{\rho}$ (so that the $k=0$ term vanishes) and accounting for the Lipschitz map that recovers $\mathbf{X}$ from $\mathbf{Z}$, we conclude
\begin{align*}
    \mathcal{W}_2^2(\mathbf{\Psi}^k_{h}\Pi_\rho, \bar{\mathbf{\Phi}}^k_h\Pi_\rho) &\leq \left(1 + \rho^{-4}\lVert\mathbf{Q}^{-1} \rVert^2_\mathrm{op} \right) \mathbb{E}\left[\lVert \Delta \widetilde{\mathbf{Z}}_{t_{k}} \rVert^2_2 +\lVert \Delta \widetilde{\mathbf{V}}_{t_{k}} \rVert^2_2 \right]\\
    &\leq 14\left(1 + \rho^{-4}\lVert\mathbf{Q}^{-1} \rVert^2_\mathrm{op} \right) C_5\kappa h.
\end{align*}
By the triangle inequality and the contraction property of the discrete update in Theorem \ref{theorem1},
\begin{align*}
\begin{split}
    \mathcal{W}_2 \left(\Pi_{\rho h}, \Pi_{\rho} \right) &\leq \mathcal{W}_2\left(\bar{\mathbf{\Phi}}_h^k\Pi_{\rho,h},  \bar{\mathbf{\Phi}}_h^k\Pi_{\rho}\right) + \mathcal{W}_2\left(  \bar{\mathbf{\Phi}}_h^k\Pi_{\rho}, \mathbf{\Psi}^k_{h}\Pi_\rho\right)\\ 
    &\leq \left( 1- \frac{h}{\kappa} + \frac{h^2}{\kappa^2} \right)^{\frac{k}{2}}\sqrt{7\left(1+\rho^{-4}\lVert\mathbf{Q}^{-1} \rVert^2_\mathrm{op} \right)}\mathcal{W}_2 \left(\Pi_{\rho h}, \Pi_{\rho} \right) +  \sqrt{14\left(1 + \rho^{-4}\lVert\mathbf{Q}^{-1} \rVert^2_\mathrm{op} \right) C_5\kappa h} 
\end{split}
\end{align*}

Choosing $k$ large enough that $\left( 1- \frac{h}{\kappa} + \frac{h^2}{\kappa^2} \right)^{\frac{k}{2}}\sqrt{7\left(1+\rho^{-4}\lVert\mathbf{Q}^{-1} \rVert^2_\mathrm{op} \right)}\leq \frac{1}{2}$, we finally obtain
\begin{align}
\label{eqn:biasestimate}
    \mathcal{W}_2 \left(\Pi_{\rho h}, \Pi_{\rho} \right) \leq 2 \sqrt{14\left(1 + \rho^{-4}\lVert\mathbf{Q}^{-1} \rVert^2 \right) C_5\kappa h} = \sqrt{C_{\mathrm{bias}}h}.
\end{align}
This concludes the proof of the bias lemma.
\end{proof}

\begin{figure*}[t]
    \centering
    \includegraphics[width=\textwidth]{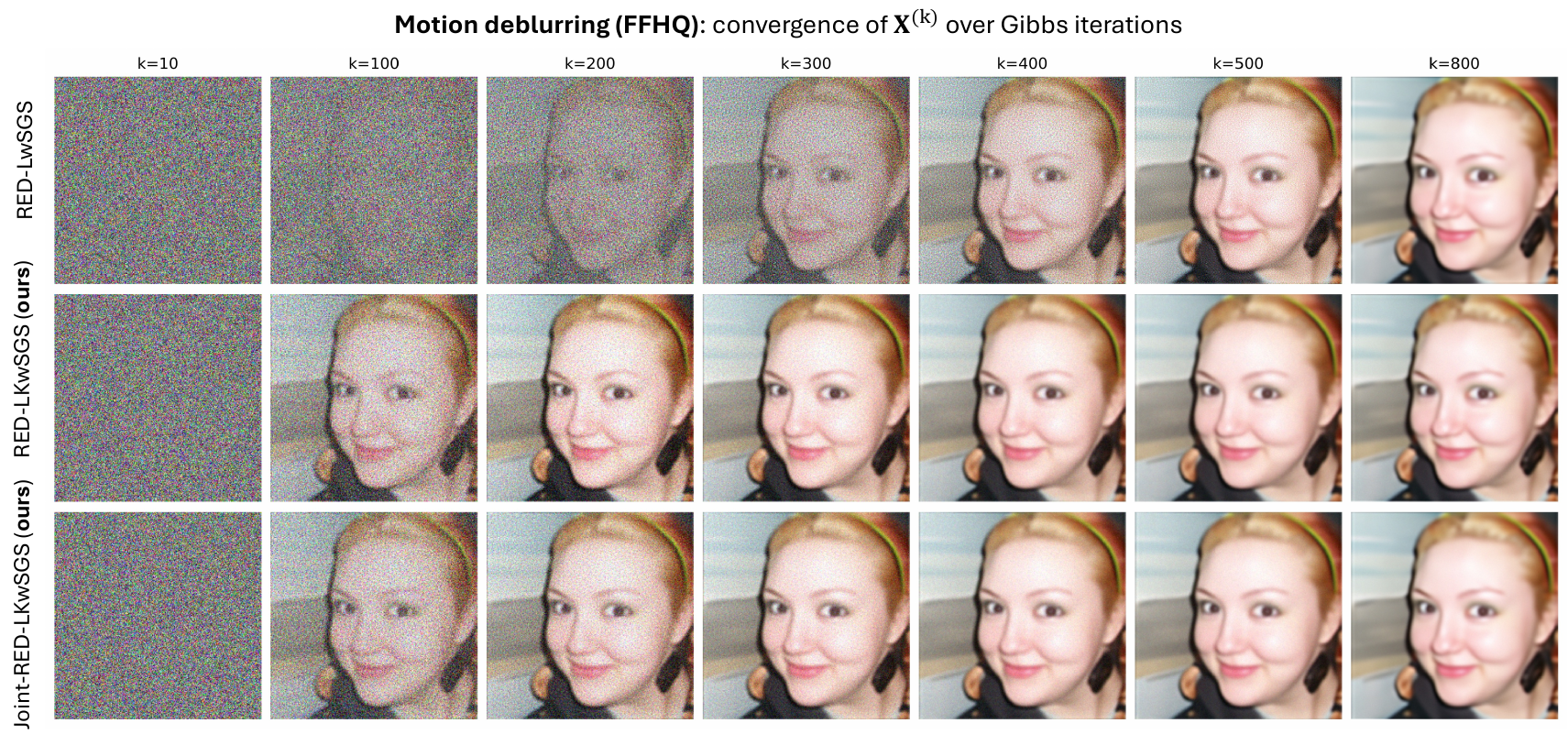}
    \caption{Visual convergence of $\mathbf{X}^{(k)}$ on motion deblurring of an FFHQ image, over 800 Gibbs iterations.}
    \label{fig:ffhq_motion}
\end{figure*}

\begin{figure*}[t]
    \centering
    \includegraphics[width=\textwidth]{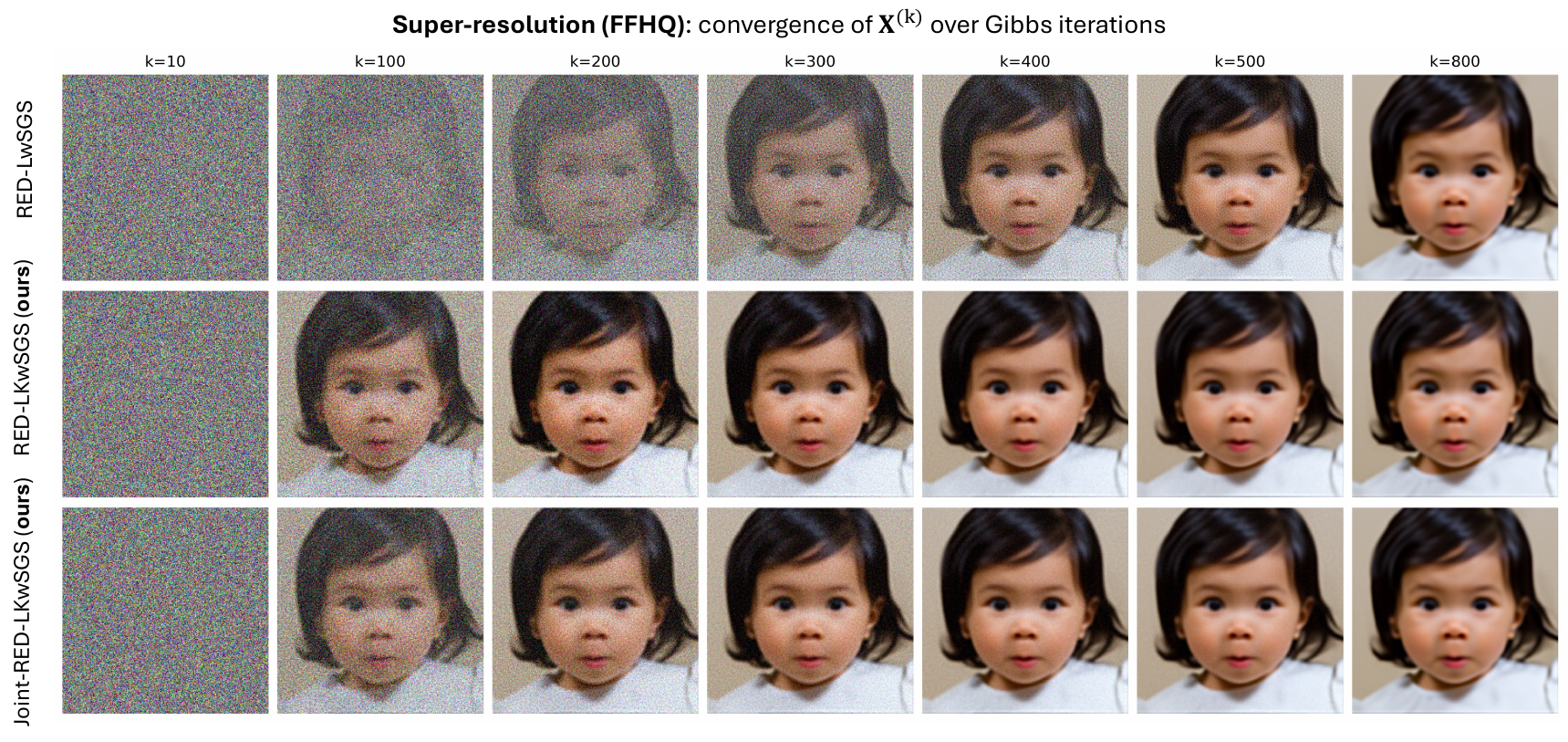}
    \caption{Visual convergence of $\mathbf{X}^{(k)}$ on super-resolution of an FFHQ image, over 800 Gibbs iterations. }
    \label{fig:ffhq_superresolutio}
\end{figure*}

\section{Proof of the Convergence Theorem \ref{theorem:convergence}}
\label{proof:convergence}
Combining the discrete contraction with the bias estimate \eqref{eqn:biasestimate} yields
\begin{align}
    \begin{split}
        \mathcal{W}_2(\bar{\mathbf{\Phi}}_h^k\mu, \Pi_\rho) &\leq \mathcal{W}_2(\bar{\mathbf{\Phi}}_h^k\mu, \Pi_{\rho h}) +\mathcal{W}_2(\Pi_{\rho h}, \Pi_\rho)\\
        & \leq \sqrt{C}\left(1-\frac{h}{\kappa} + \frac{h^2}{\kappa^2} \right)^{k/2}\mathcal{W}_2(\mu, \Pi_{\rho h}) + \sqrt{C_\mathrm{bias} h}.
    \end{split}
\end{align}
Choose the step size $h>0$ small enough such that $  \sqrt{C_\mathrm{bias} h} \leq\delta/2$, and then choose the number of iterations 
\begin{equation*}
    k \geq \frac{2 \ln (2\sqrt{C}(\mathcal{W}_2(\mu, \Pi_\rho)+\delta/2)/\delta)}{\ln \left(1/\left(1-\frac{h}{\kappa} + \frac{h^2}{\kappa^2} \right)\right)}
\end{equation*}
so that the contraction term is at most $\delta/2$. The resulting bound $\mathcal{W}_2(\bar{\mathbf{\Phi}}_h^k\mu, \Pi_\rho)\leq \delta$ is the claimed convergence guarantee.

\section{Experimental Details}
\label{app:additionalresults}
\subsection{Image restoration tasks}

\paragraph{Common setting.}
We use $256\times256$ RGB images from FFHQ and ImageNet, normalized to $[0,1]$, so that $d=3\times256^2=196{,}608$.
All forward operators act independently on each color channel. For every task,
\begin{equation*}
\mathbf{y}=\mathbf{A}\mathbf{x}+\mathbf{n},\qquad \mathbf{n}\sim\mathcal{N}(\mathbf{0},\mathbf{\Omega}^{-1}),\qquad
p(\mathbf{y}\mid\mathbf{x})\propto\exp\Big[-\tfrac12(\mathbf{A}\mathbf{x}-\mathbf{y})^\top\mathbf{\Omega}(\mathbf{A}\mathbf{x}-\mathbf{y})\Big],
\end{equation*}
where $\mathbf{A}\in\mathbb{R}^{m\times d}$ is the forward operator and the noise covariance $\mathbf{\Omega}^{-1}$ is diagonal.
We follow the experimental setting of \citet{coeurdoux2024plug}. For a given image, the observation $\mathbf{y}$
(including the noise realization) is generated once with a fixed seed and is shared by all methods.

\paragraph{Gaussian deblurring.}
The operator $\mathbf{A}$ is a circulant convolution matrix: $\mathbf{A}\mathbf{x}=\mathcal{K}_G*\mathbf{x}$ under circular boundary
conditions, with the isotropic Gaussian kernel
\begin{equation*}
\mathcal{K}_G(u,v)\ \propto\ \exp\Big(-\frac{u^2+v^2}{2\sigma_k^2}\Big),\qquad \sigma_k=3.0,
\end{equation*}
defined on a $61\times61$ window and normalized to sum to one. The kernel is symmetric and blurs all directions equally.
Here $m=d$. The noise is spatially varying: the covariance is diagonal,
$\mathbf{\Omega}^{-1}=\mathrm{diag}[\sigma_1^2,\dots,\sigma_d^2]$, with the $\sigma_i$ drawn independently,
\begin{equation*}
\sigma_i=0.157\ \text{with probability }0.35,\qquad \sigma_i=0.051\ \text{with probability }0.65 .
\end{equation*}

\paragraph{Motion deblurring.}
The operator is of the same type, $\mathbf{A}\mathbf{x}=\mathcal{K}_M*\mathbf{x}$, with a non-symmetric motion kernel $\mathcal{K}_M$ of
size $61\times61$. The kernel is the rasterized trace of a random camera path with angle $\theta_t=\theta_{t-1}+\epsilon_t$,
$\epsilon_t\sim\mathcal{N}(0,\alpha^2)$ and intensity $\alpha=0.5$, normalized to sum to one. It is anisotropic: it smears the
image strongly along the direction of motion and weakly across it. The noise is as in Gaussian deblurring.

\paragraph{Super-resolution.}
The forward operator combines a blur and a subsampling step, $\mathbf{A}=\mathbf{S}\mathbf{T}$. Here $\mathbf{T}$ is a circulant
convolution with a Gaussian kernel of size $9\times9$ and standard deviation $1.5$, and $\mathbf{S}$ subsamples the blurred image
by a factor $4$ in each direction, so that a $256\times256$ image is observed as $64\times64$ and $m=3\times64^2=12{,}288<d$.
The noise is white, $\mathbf{\Omega}^{-1}=\sigma^2\mathbf{I}_m$, with $\sigma^2$ chosen for each image to reach an SNR of $40$\,dB.

\paragraph{Update of $\mathbf{X}$.}
For all tasks the data-side conditional is Gaussian, $p(\mathbf{x}\mid\mathbf{z})=\mathcal{N}(M(\mathbf{z}),\mathbf{Q}^{-1})$, with
\begin{equation*}
\mathbf{Q}=\mathbf{A}^\top\mathbf{\Omega}\mathbf{A}+\frac{1}{\rho^2}\mathbf{I},\qquad
M(\mathbf{z})=\mathbf{Q}^{-1}\Big(\mathbf{A}^\top\mathbf{\Omega}\mathbf{y}+\frac{1}{\rho^2}\mathbf{z}\Big).
\end{equation*}
Thanks to the splitting, this step does not depend on the prior $g$ and boils down to a high-dimensional Gaussian sampling task.
To avoid the computation of
$\mathbf{Q}^{-1}$, we use the exact perturbation--optimization algorithm \citep{marnissi2018auxiliary}. We draw
$\boldsymbol{\xi}_1\sim\mathcal{N}(\mathbf{0},\mathbf{I}_m)$ and $\boldsymbol{\xi}_2\sim\mathcal{N}(\mathbf{0},\mathbf{I}_d)$ and set
$\mathbf{X}$ to the solution of
\begin{equation*}
\mathbf{Q}\,\mathbf{X}=\mathbf{A}^\top\mathbf{\Omega}\mathbf{y}+\frac{1}{\rho^2}\mathbf{z}
+\mathbf{A}^\top\mathbf{\Omega}^{1/2}\boldsymbol{\xi}_1+\frac{1}{\rho}\boldsymbol{\xi}_2 .
\end{equation*}
The random part of the right-hand side has covariance $\mathbf{A}^\top\mathbf{\Omega}\mathbf{A}+\rho^{-2}\mathbf{I}=\mathbf{Q}$, so $\mathbf{X}$ has
mean $M(\mathbf{z})$ and covariance $\mathbf{Q}^{-1}\mathbf{Q}\mathbf{Q}^{-1}=\mathbf{Q}^{-1}$: the draw is exact (up to the tolerance of the
linear solver) and only requires solving one linear system with matrix $\mathbf{Q}$. This $\mathbf{X}$-update is used by all
split Gibbs samplers compared in the paper except Joint-RED-KLwSGS. The latter does not sample $\mathbf{X}$ exactly: it updates
$\mathbf{X}$ with kinetic Langevin diffusions and only needs the gradient of the data-fidelity term,
$\nabla_{\mathbf{x}}f(\mathbf{x},\mathbf{y})=\mathbf{A}^\top\mathbf{\Omega}(\mathbf{A}\mathbf{x}-\mathbf{y})=-\nabla_{\mathbf{x}}\log p(\mathbf{y}\mid\mathbf{x})$,
which requires applying $\mathbf{A}$ and $\mathbf{A}^\top$ but never forming or solving with $\mathbf{Q}$.

\begin{figure*}[t]
    \centering
    \includegraphics[width=\textwidth]{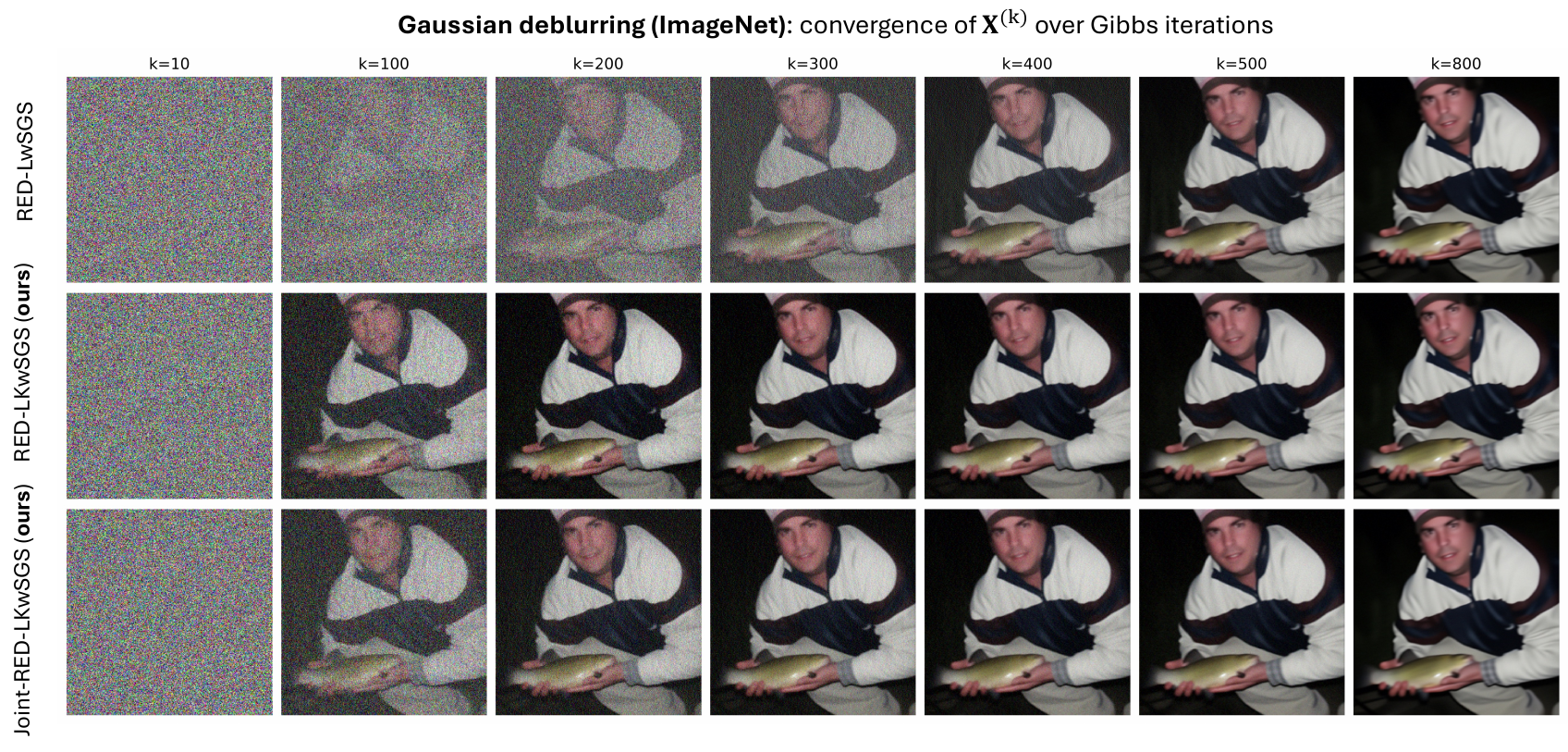}
    \caption{Visual convergence of $\mathbf{X}^{(k)}$ on Gaussian deblurring of an ImageNet image, over 800 Gibbs iterations. }
    \label{fig:imagenet_gaussian}
\end{figure*}

\subsection{Implementation details of RED-LwSGS, RED-KLwSGS and Joint-RED-KLwSGS}
\paragraph{Diffusion models.} We use the same pretrained diffusion models as PnP-SGS: the models of \citet{choi2021ilvr} and  \citet{dhariwal2021diffusion}  for FFHQ and ImageNet, respectively. The two models differ by their noise schedules $\alpha_t
$, and we use the schedules as PnP-SGS \citep{faye2024regularization}.

\paragraph{Coupling parameter.} The coupling parameter $\rho$ defines the augmented distribution $\pi_\rho$. For all three methods and both datasets, we use a decreasing sequence of $\rho$, with initial value 0.5 and decay rate 0.85. 

\paragraph{Iterations and estimator.} For all experiments, the three methods run $N_{MC}=800$ iterations, including $N_{bi}=20$ burn-in iterations. From the iterates $\left\{\mathbf{X}^{(k)},\mathbf{Z}^{(k)} \right\}_{k=1}^{N_{MC}}$, Bayesian quantities, such as posterior mean and credibility intervals, can be approximated. After convergence, the iterates $\left\{\mathbf{X}^{(k)} \right\}_{k=1}^{N_{MC}}$ are approximately marginally distributed according to $\pi_\rho(\mathbf{x})$, so the minimum mean squared estimator (MMSE or posterior mean) $\hat{\mathbf{X}}_{\mathrm{MMSE}}=\mathbb{E}[\mathbf{x}|\mathbf{y}]$ can be approximated by the empirical mean after burn-in,
\begin{equation*}
\hat{\mathbf{X}}_{\mathrm{MMSE}}\approx\frac{1}{N_{\mathrm{MC}}-N_{bi}}\sum_{k=N_{\mathrm{bi}}}^{N_{\mathrm{MC}}}\mathbf{X}^{(k)}.
\end{equation*}
\paragraph{Step size and velocity reset.} We fix $\gamma=2$ for RED-KLwSGS and Joint-RED-KLwSGS and  the step size $h=0.001$ for all three methods. For momentum-based methods
RED-KLwSGS and Joint-RED-KLwSGS, we observe oscillations of the quality metrics over iterations. To prevent them, we reset the velocity $\mathbf{V}$ (for RED-KLwSGS) and both momenta $\mathbf{U}$ and $\mathbf{V}$ (for Joint-RED-KLwSGS) to zero whenever the PSNR stops improving. The PSNR is computed for the running average of the five most recent iterates,
\begin{equation*}
    \hat{\mathbf{X}}^{(k)}_{5}=\frac{1}{5}\sum_{i=k-4}^{k}\mathbf{X}^{(i)}.
\end{equation*}

\paragraph{Additional results.}
Figures \ref{fig:ffhq_motion},\ref{fig:ffhq_superresolutio} show the visual convergence on FFHQ for motion deblurring and super-resolution. Figures \ref{fig:convergence_imagenet}, \ref{fig:imagenet_gaussian}, \ref{fig:imagenet_motion}, \ref{fig:imagenet_superresolution} show the quantitative convergence, visual convergence on ImageNet for Gaussian deblurring, motion deblurring and super-resolution.


\begin{figure*}[t]
    \centering
    \includegraphics[width=\textwidth]{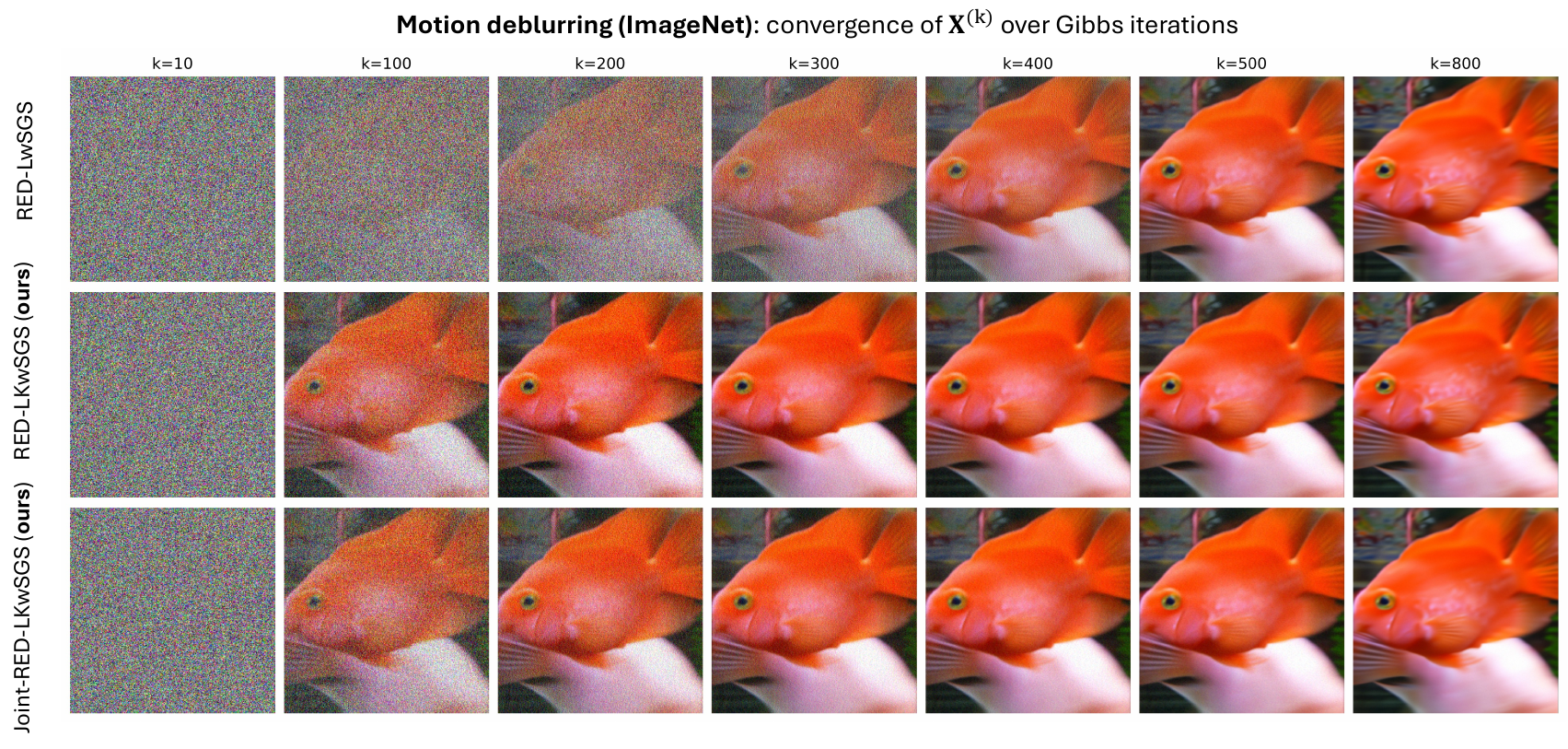}
    \caption{Visual convergence of $\mathbf{X}^{(k)}$ on motion deblurring of an ImageNet image, over 800 Gibbs iterations. }
    \label{fig:imagenet_motion}
\end{figure*}

\begin{figure*}[t]
    \centering
    \includegraphics[width=\textwidth]{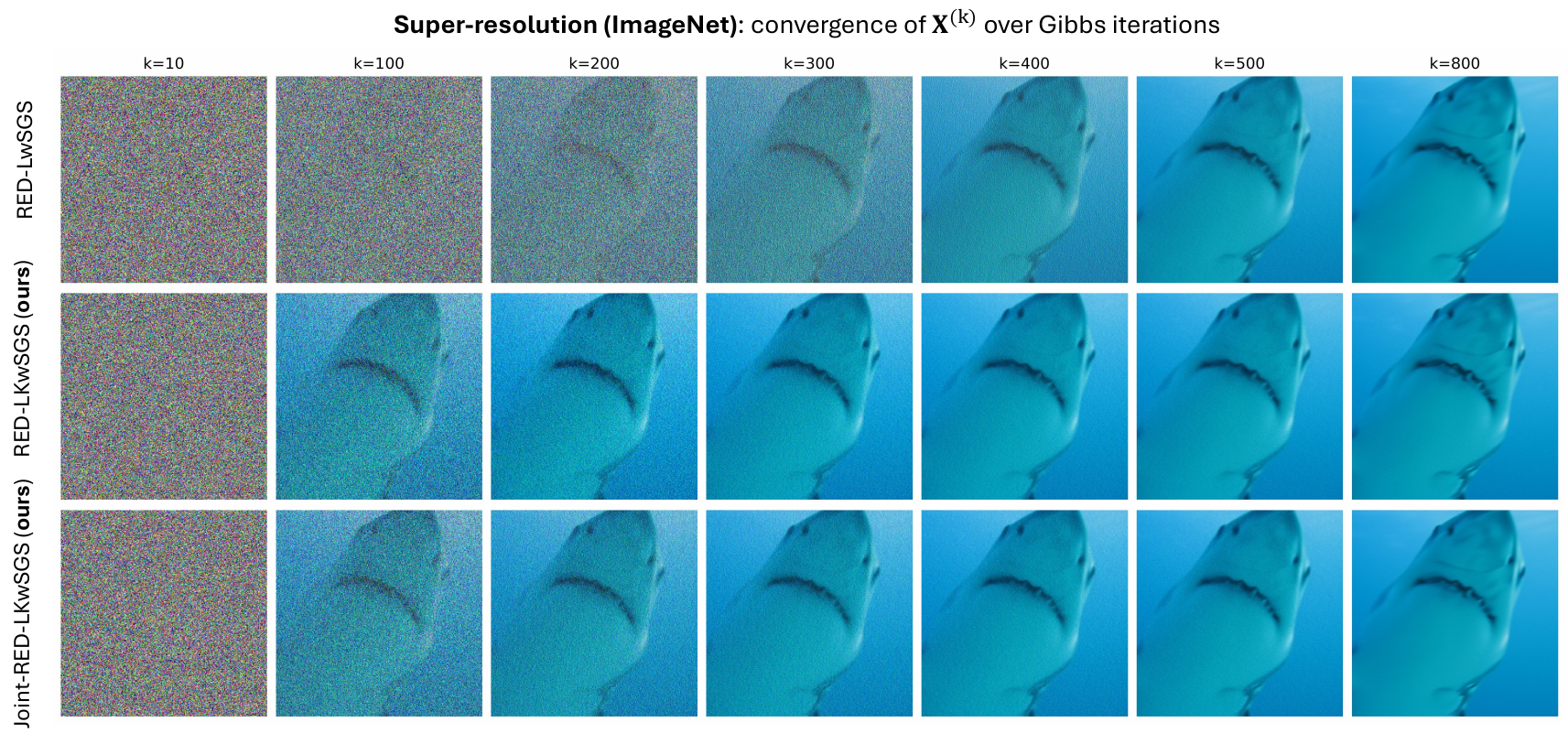}
    \caption{Visual convergence of $\mathbf{X}^{(k)}$ on super-resolution of an ImageNet image, over 800 Gibbs iterations. }
    \label{fig:imagenet_superresolution}
\end{figure*}

\end{document}